\documentclass[HARVARD,Times1COL]{WileyNJDv5}
\usepackage{comment}
\usepackage{subfigure}
\usepackage{float}
\usepackage{soul} % for the command \hl
\articletype{Original Article}%
\usepackage{tikz}
\usetikzlibrary{arrows.meta, decorations.pathreplacing}
\received{Date Month Year}
\revised{Date Month Year}
\accepted{Date Month Year}
\journal{Submitted to the Scandinavian Journal of Statistics}
\volume{00}
\copyyear{2024}
\startpage{1}

\usepackage{amssymb}
\DeclareMathOperator\sk{sk}
\DeclareMathOperator\cl{cl}
\DeclareMathOperator\an{an}

\newcommand{\fla}{\mbox{$\hspace{.05em} \prec
\!\!\!\!\!\frac{\nn \nn}{\nn}$}}
\newcommand{\arc}{\mbox{$\hspace{.06em} \prec
\!\!\!\!\!\frac{\nn \nn}{\nn}%\frac{\nn \nn}{\nn }
\!\!\!\!\!
\succ\! \hspace{.25ex}$}}
\newcommand{\ts}[1]{\textcolor{blue}{*** TS: #1 ***}} %%
\newcommand{\erk}{\hfill \ensuremath{\Diamond}} %%used to signal end of remark
\newcommand\ns[1]{ \left\{ {#1} \right\} }  %%set builder notation
\newcommand\vij{V\setminus\ns{i,j}}
\newcommand\vijk{V\setminus\ns{i,j,k}}
\newcommand{\ful}{\mbox{$\, \frac{ \nn \nn \;}{ \nn \nn
}$}}
\newcommand{\nn}[0]{\hspace*{.7em}}
\newcommand{\fra}{\mbox{$\hspace{.05em} \frac{\nn
\nn}{\nn
}\!\!\!\!\! \succ \! \hspace{.25ex}$}}
\usepackage{float}
\begin{document}

\title{A General Framework on Conditions for Constraint-based Causal Learning}

\author[1]{Kai Z. Teh}

\author[1]{Kayvan Sadeghi}

\author[1]{Terry Soo}

%\authormark{Kai Teh \textsc{et al.}}
\authormark{Teh \textsc{et al.}}
%\titlemark{A General Framework for Constraint-based Causal Learning}
\titlemark{Framework for Constraint-based Learning}

%\address[1]{\orgdiv{Department of Statistical Science}, 
\address{\orgdiv{Department of Statistical Science}, 
\orgname{University College London}, \orgaddress{\state{London}, \country{UK}}}

\begin{comment}
    \address[2]{\orgdiv{Department Name}, \orgname{Institution Name}, \orgaddress{\state{State Name}, \country{Country Name}}}

\address[3]{\orgdiv{Department Name}, \orgname{Institution Name}, \orgaddress{\state{State Name}, \country{Country Name}}}
\end{comment}

%\corres{Kai Teh, This is sample corresponding address.
\corres{Kai Teh \email{kai.teh.21@ucl.ac.uk}}

%\presentaddress{This is sample for present address text this is sample for present address text.}

%\fundingInfo{Kayvan grant?}
%\JELinfo{ejlje}

\abstract[Abstract]{Most constraint-based causal learning algorithms provably return the correct causal graph under certain correctness conditions, such as faithfulness. By representing any constraint-based causal learning algorithm using the notion of a property, %we decompose the correctness condition into a part relating the distribution and the true causal graph,  and a part that depends solely on the distribution.  
we provide a general framework to obtain and study correctness conditions for these algorithms. From the framework, we provide exact correctness conditions for the PC algorithm, which are then related to the correctness conditions of some other existing causal discovery algorithms. The framework also suggests a paradigm for designing causal learning algorithms which allows for the correctness conditions of algorithms to be controlled for before designing the actual algorithm, and has the following implications. We show that the sparsest Markov representation condition is the weakest correctness condition for algorithms that output ancestral graphs or directed acyclic graphs satisfying any existing notions of minimality.
%We show that the sparsest Markov representation condition is the weakest correctness condition resulting from existing notions of minimality for maximal ancestral graphs and directed acyclic graphs.  
We also reason that Pearl-minimality is necessary for meaningful causal learning but not sufficient to relax the faithfulness condition and, as such, has to be strengthened, such as by including background knowledge, for causal learning beyond faithfulness.}

\keywords{Causal Discovery; Graphical Models}

%\jnlcitation{\cname{%
%\author{Teh K},
%\author{Sadeghi K} and
%\author{Soo T}}.
%\ctitle{A General Framework for Constraint-based Causal Learning} \cjournal{\it } \cvol{}.}

\maketitle

%\renewcommand\thefootnote{}
%\footnotetext{\textbf{Abbreviations:} ANA, anti-nuclear antibodies; APC, antigen-presenting cells; IRF, interferon regulatory factor.}
%\renewcommand\thefootnote{\fnsymbol{footnote}}
%\setcounter{footnote}{1}

\section{Introduction}\label{sec1}

%There has always been of great interest in different fields to infer causal relationships, with frameworks like potential outcomes and graphical models gaining prominence amongst the causality community. 

A main goal of graph-based causal inference is causal discovery---given data, one would like to uncover the underlying causal structure in the form of a true causal graph, on which conventional graph-based causal inference techniques depend. We will mostly be concerned with the setting of observational data,
%only, 
such as when randomised control trials are unavailable.  In the absence of interventional data, the true causal graph is only recoverable up to its graphical separations. 
%Broadly, existing 
Causal discovery approaches can be categorised into score-based approaches \citep{score} and constraint-based approaches \citep{zbMATH01600338}, of which the latter will be the focus of this work.  

%To avoid making vacuous causal statements, untestable assumptions are necessary for constraint-based approaches to exclude erroneous causal structures which do not best represent the data. 
%\ts{\textcolor{red}{look}do you think this is really the correct order as to what is happening?  is this what is really happening with the assumptions, or are the assumptions sort of making a model selection?}

To work correctly, constraint-based approaches require the (untestable) assumption and condition that the probabilistic dependency structure is a good representation for the graphical structure of the true causal graph.  
A natural and common assumption is 
%that of
faithfulness,  where every conditional independence in the data generating distribution is exactly represented by the true causal graph \citep{Faith}.  Under faithfulness, most constraint-based learning approaches such as 
the PC and SGS algorithms provably return the true causal graph, up to its graphical separations. However, simple examples show that faithfulness is easily violated, and  can be too stringent in practice and theory \citep{faithbad,Andersen_2013}.

To relax the faithfulness assumption, there have been proposals of causal learning algorithms with relaxed correctness conditions. These algorithms include the sparsest permutation (SP) algorithm by \citet{UhlSP} and the recent natural structure learning algorithms of \citet{Sad} and \citet{teh}, both of which provably return the graphical separations of the true causal graph under strictly weaker assumptions than faithfulness.

Addressing the problem of conditions for causal learning algorithm being too strong, we introduce a theoretical framework which encompasses the conditions under which any constraint-based causal learning algorithms work, along with providing such conditions given any algorithm. Via the notion of a property, the framework shows a duality between algorithms and conditions under which the algorithms work. This duality also suggests an alternative paradigm for designing causal learning algorithms where just by considering the related property, the correctness condition of the algorithm is controlled for before designing the actual computational steps of the algorithm itself.  %required for the designed algorithm is not too strong, 
As implications of the provided framework we will show the following:  1.) we provide exact conditions for when the PC algorithm works, and relate them to correctness conditions of some other existing algorithms;  2.) we show that if one 
were 
to attempt to build an algorithm that outputs maximal ancestral graphs or directed acyclic graphs that satisfy any existing notions of minimality, the sparsest Markov representation assumption from \cite{UhlSP} is the weakest possible condition for such algorithms to work;
%the sparsest Markov representation condition is the weakest correctness condition resulting from existing notions of minimality for maximal ancestral graphs and directed acyclic graphs;  
3.) we reason that Pearl-minimality \citep{zbMATH05645279} is necessary for meaningful causal learning, and strengthening Pearl-minimality, such as by including background knowledge \citep{meek}, is necessary for meaningful causal learning beyond faithfulness.

The structure of the paper is as follows: Section \ref{sec2} covers the relevant background, Section \ref{theory} covers the theory of the main framework, and Sections \ref{type} and \ref{degen} discuss implications and applications of the framework. 
%All 
The proofs are given in the appendix.

\section{Preliminaries}
\label{sec2}

Here, we will cover all the concepts and terminology relevant for this work. 

\subsection{Graphical Models}
Let \(G\) denote a graph over a finite  set of nodes \(V = \ns{1, \ldots, n} \), with 
only
one of the three types of edges:  directed  
($\fra$),   
undirected 
(\(\ful\)), and 
bidirected edges 
($\arc$), connecting any two \emph{adjacent} nodes. A \emph{path} \(\pi\) between nodes \(i_0\) and \(i_n\) is a sequence of nodes 
\(\langle i_0,\ldots, i_n \rangle\), 
such that for all \(m \in \ns{0\ldots,n-1}\), \(i_m\) is adjacent to \(i_{m+1}\);
%, i.e.\ there exists an edge between \(i_m\) and \(i_{m+1}\); 
if, in addition,  we have \(i_0=i_n\), then the path is a  \emph{cycle}. 
 For a path \(\langle i_0,\ldots, i_n \rangle\), if all edges between nodes \(i_m\) and \(i_{m+1}\) are directed as \(i_m\fra i_{m+1}\) %for \(m=0,\ldots,n-1\), 
the path is a \emph{directed path}; likewise, for a cycle \(\langle i_0,\ldots, i_n \rangle\), if all edges between nodes \(i_m\) and \(i_{m+1}\) are directed as \(i_m\fra{} i_{m+1}\),  %for \(m=0,\ldots,n-1\), 
then the cycle is a \emph{directed cycle}. Given \(C\subseteq V\), we let \(\an(C)\) denote the \emph{ancestors} of \(C\), the set of nodes such that there exists a directed path to some node \(i\in C\).  

The most general class of graphs we consider are \emph{anterial} graphs.

\begin{definition}[Anterial graphs \citep{lkayvan}]
    A graph \(G\) over a set of  nodes \(V\), which may contain directed,  %(\(\rightarrow\)),  
    undirected, 
    %( --- ), 
    %
    and bidirected edges 
    %(\(\leftrightarrow\))
    is an \emph{anterial} graph if the following omissions are satisfied.
    %the following do \emph{not} hold: 
    
    \begin{enumerate}
        \item 
        There does not exist a path \(\langle i_0,\ldots, i_n\rangle \) such that \(i_0\arc i_n\) and for all \(m \in \ns{0,\ldots,n-1}\) the  edge between nodes \(i_m\) and \(i_{m+1}\) 
        is either undirected 
        (\(\ful\)) 
        or directed 
        as \(i_m\fra i_{m+1}\).
        \item 
        There does not exist
        a cycle \(\langle i_0, \ldots, i_n\rangle \) 
        %\(i_0=i_n\), 
        such that for \(m \in \ns{0,\ldots,n-1}\) the edge between nodes \(i_m\) and \(i_{m+1}\)
        is either undirected or directed as \(i_m\fra  i_{m+1}\).
    \end{enumerate}
\end{definition}
 An \emph{ancestral} graph is an anterial graph with the constraint that there are no arrowheads pointing into undirected edges, and a \emph{directed acyclic} graph (DAGs) is an anterial graph with only directed edges.
 %\(\rightarrow\). 
 Ancestral graphs can be seen as a generalisation of a DAG causal model with unobserved variables \citep{mag}.

%admit a preorder \citep[Proposition 10]{orderfaith}, and 

\begin{definition}[Ancestral graphs \citep{10.1214/aos/1031689015} and directed acyclic graphs]
    A graph \(G\) over a set of  nodes \(V\),  which may contain directed,  undirected, 
    and bidirected edges  
    is an \emph{ancestral} graph if 
    %the following cannot hold:
    the following omissions are satisfied.
    \begin{enumerate}
        \item There does not exist nodes \(i,j,k \in V\) such that \(i\ful j\arc k\) or \(i\ful j\fla  k\).
        \item 
        There does not exist a directed path between nodes \(i\) and \(j\) such that \(i\arc j\).
        \item 
        There are no directed cycles in the graph.
    \end{enumerate}
   %\ts{wasn't this defined above?} A graph \(G\) over a set of  nodes \(V\), which contains only directed edges is a \emph{directed acyclic} graph if there are no directed cycles in the graph.
\end{definition}
For \(A,B,C \subseteq V\) 
disjoint,  we let 
\(A\perp_G B\cd C\) denote  
a 
graphical separation in $G$,  between $A$ and $B$ given \(C\). Anterial graphs have a well-defined graphical separation \citep{lkayvan}, which specialises to the classical d-separation \citep{zbMATH05645279} in the case of DAGs, 
and m-separation \citep{10.1214/aos/1031689015} in the case of ancestral graphs.
We associate  a joint distribution \(P\) to the set of nodes \(V\), and a  random vector  \(X=(X_1,\ldots,X_n)\) with the distribution $P$.   We let  \(A\ci B\cd C\) denote 
the  
conditional independence of \( (X_i)_{i\in A}\) and \((X_j)_{j\in B}\) given
\((X_k)_{k\in C}\), which can be thought of as probabilistic separation.

Let \(J(P)\) denote the set triples corresponding to conditional independencies of distribution \(P\), so that  $(A,B,C) \in J(P)$ if and only if \(A\ci B\cd C\).   Similarly, let \(J(G)\) be the set of triples corresponding to  graphical separations of graph \(G\). 
The graphs 
\(G_1\) and \(G_2\) are  \emph{Markov equivalent} and belong in the same \emph{Markov equivalence class (MEC)} if \(J(G_1)=J(G_2)\). Throughout this work, from $P$, we will only be making use of the conditional independence structure captured by $J(P)$, thus we may refer to \(P\) and \(J(P)\) interchangeably.   

\subsection{General Setting for Causal Learning}
In constraint-based causal learning, we often restrict our attention to certain graph classes, such as DAGs for the PC algorithm or ancestral graphs for the FCI algorithm \citep{fci} as such we let \(\mathbb{G}\) denote the class of graphs being considered for causal learning, such as DAGs, ancestral graphs or anterial graphs.

Let $G_0\in \mathbb{G}$ be the true causal graph, which we wish to partially recover  from an observed   distribution \(P\), which is induced from $G_0$. 
The main goal of causal learning is to use observational data from the distribution \(P\) to recover graphs that belong to the MEC of \(G_0\). 
As with constraint-based causal learning, from \(P\) we will mostly be concerned with the conditional independencies \(J(P)\) and  by tacitly assuming the availability of a \emph{conditional independence oracle}, 
%i.e., assuming 
%we have an oracle that 
we will always 
%correctly tell us 
know
whether or not a given conditional independence relation holds in the distribution.

A causal learning algorithm aims to output a set of graphs \(\boldsymbol{G}(P)\subseteq \mathbb{G}\) from an input distribution \(P\), and if all the graphs in the output set \(\boldsymbol{G}(P)\) are Markov equivalent to the true causal graph \(G_0\), then the algorithm is \emph{correct}. The condition under which the algorithm is correct is then the \emph{correctness condition} of the algorithm.  As in \citet{Sad}, this paradigm can be summarised diagrammatically:
\begin{center}
     \begin{tikzpicture}[]
\node (e1) at (0,-1.5) {True causal graph \(G_0\in \mathbb{G}\)};
\node (e2) at (0,0) {Observational distribution \(P\)};
\node (e3) at (7,0) {\(J(P)\)};
\node (e4) at (7,-1.5) {Output \(\boldsymbol{G}(P)\subseteq \mathbb{G}\)};
\draw[->] (e1) -- (e2) node[midway, left] { induce};
\draw[->] (e2) -- (e3) node[midway, above] { induce};
\draw[->] (e3) -- (e4) node[midway, right] { algorithm};
\draw[-, draw=none] (e1) -- (e4) node[midway] { \(\sim\)} node[midway, above] { Markov equivalent};
\node (e1) at (0,0) {};
 %\draw[decorate,decoration={brace,amplitude=10pt,mirror}] (-6.9,-0.2) -- (2.75,-0.2) node[midway, yshift=-0.5cm] {Informed by Theorem \ref{mainth}};
    \draw[decorate,decoration={brace,amplitude=10pt,mirror}] (2.5,-1.6) -- (5,-1.6)
    node[midway, yshift=-0.5cm] {If correctness condition holds};
    \end{tikzpicture}
\end{center}

\subsection{Assumptions in Causal Learning Literature}\label{assume}
Here, we will introduce 
relations tying together graphical separations of a graph \(G\) and conditional independencies of a distribution \(P\).
Often these relations are assumed on the true causal graph and its observational distribution, and the success of a learning algorithm leans heavily on the  assumed relation.  Thus in the context of structure learning, these relations are often assumptions. Note that throughout we are merely stating the assumptions in the literature, not necessarily assuming them to be true here. The most fundamental relation  is the \emph{Markov property}.
\begin{definition}[Markov property]
   A distribution  \(P\) is \emph{Markovian} to \(G\) if  $J(G) \subseteq J(P)$---equivalently,  for all disjoint $A,B,C \subseteq V$,  we have 
   $$A\perp_G B\cd C \Rightarrow A\ci B\cd C.$$ 
\end{definition}
If \(P\) is induced from a 
structural equation model 
for a 
DAG \(G\) with independent noise, then the Markov property is satisfied. 
%The Markov property is often proven shown via a different version of it, known as the \emph{pairwise Markov property}.
The Markov property has many related forms, and the following pairwise version along with additional assumptions implies the Markov property.
\begin{definition}[Pairwise Markov Property]
    A distribution \(P\) is \emph{pairwise Markovian} to \(G\) if for all nodes \(i,j\in V\), we have
    $$i \textnormal{ not adjacent to } j \textnormal{ in } G \Rightarrow i\ci j\cd \an(i,j).$$
\end{definition}
If, in addition to the Markov property, we have the reverse implication, then we have \emph{faithfulness}, one of the strongest assumptions in causal learning \citep{zbMATH05645279,zbMATH01600338}.
\begin{definition}[Faithfulness]
     A distribution \(P\) is \emph{faithful} to \(G\) if $J(G) = J(P)$---equivalently,   for all disjoint \(A,B,C \subseteq V\), we have
     $$A\perp_G B\cd C \iff A\ci B\cd C.$$ 
\end{definition}
\begin{remark}
    There are stronger assumptions such as the \(\lambda\)-strong faithfulness \citep{strongfaith}, which essentially bounds the strength of the conditional dependence, in the sense of the magnitude of the conditional  Pearson correlation coefficient, to be above \(\lambda\), in order to achieve consistency in conditional independence testing. Since we will not be concerned with testing in  this work, we will not include assumptions of this kind.  \erk
\end{remark} 

Let \(\sk(G)\) denote the \emph{skeleton} of graph \(G\), formed by removing all arrowheads from edges in \(G\). A \emph{v-configuration} is a set of 
three nodes \(i,k,j\in V\) 
such that \(i\) and \(j\) 
are adjacent to \(k\), but \(i\) and \(j\) are not adjacent, 
and will be represented as \(i\sim k\sim j\). A v-configuration oriented as \(i\fra k\fla j\) is a \emph{collider}, otherwise the v-configuration is a \emph{non-collider}. To relate \(P\) with \(\sk(G)\), we have the following.
\begin{definition}[The skeleton \(\sk(P)\)]
    Given a distribution \(P\), the skeleton \(\sk(P)\) is the undirected graph with node set \(V\), such that for all \(i,j\in V\), we have \(i\) is adjacent to \(j\) if and only if there does not exist any \(C\subseteq V\backslash\{i,j\}\) such that \(i\ci j\cd C\).
\end{definition}

\begin{definition}[Adjacency faithfulness]
 A distribution   \(P\) is \emph{adjacency faithful} 
  with respect to  (w.r.t.) 
    \(G\) if \(\sk(P)=\sk(G)\).
\end{definition}

Building on top of the Markov assumption, we have the following variations, all based on some notion of minimality. Based on the skeletons, \cite{orderfaith} introduced the \emph{minimally Markovian} assumption.  
\begin{definition}[Minimally Markovian]
  A distribution  \(P\) is \emph{minimally Markovian} w.r.t.\ \(G\) if \(P\) is Markovian to \(G\) and \(P\) is adjacency faithful w.r.t.\ \(G\).
\end{definition}

%We let \(\mathbb{G}\) denote a class of graphs such as DAGs, ancestral graphs or anterial graphs; we will often restrict our attention to certain graph classes. 
The following notions of minimality are originally defined in the context of DAGs, here we give natural extensions of these definitions to a more general graph class \(\mathbb{G}\), which coincides with the original definition when \(\mathbb{G}\) are DAGs.

Let \(|E(G)|\) denote the number of edges in graph \(G\). 
\begin{definition}[Sparsest Markov property \citep{forster}]\label{sparsM}
    Let $\mathbb{G}$ be a class of graphs. A graph  \(G\in \mathbb{G}\) is a \emph{sparsest Markov} graph of \(P\) if \(P\) is Markovian to \(G\) and
    \begin{equation*}
        |E(G)|\leq  |E(G')| \quad \text{for 
 all } G'\in \mathbb{G} 
 \text{ such that  } 
 P \text{ is Markovian to } 
 G'.  
    \end{equation*}
\end{definition}

The minimal graphs in Definition \ref{sparsM} refer to graphs with the least number of edges. In the case of DAGs interpreted as structural equation models (SEMs)  \citep{10.1214/aoms/1177732676}, minimal DAGs correspond to SEMs with functions that use the least overall number of arguments.  
%which intuitively represent DAGs with the lowest ‘complexity’.
%\begin{definition}[Sparsest Markov assumption] For a fixed distribution $P$ on a set of vertices $V$, suppose that  $P$ is Markovian to a graph $G$ with vertices $V$.  We say that $G$ is the \ts{a}     \emph{sparsest Markov} graph to \(P\) if whenever $P$ is Markovian to graph $G'$ with vertices $V$, then  $|E(G) \leq |E(G')|$. \end{definition}
%
\begin{definition}[Pearl-minimality assumption \citep{zbMATH05645279}]\label{p-mindef}
Let $\mathbb{G}$ be a class of graphs.  A distribution $P$ is \emph{Pearl-minimal} to a graph \(G\in \mathbb{G}\) if 
\begin{equation*}
    \nexists G'\in \mathbb{G} \text{ such that } 
    J(G)\subset J(G')\subseteq J(P).
\end{equation*}
\end{definition}
Note that the distribution \(P\) being Markovian to graph \(G\) is equivalent to \(J(G)\subseteq J(P)\), with \(J(G)=J(P)\) being equivalent to \(P\) is faithful to \(G\). Thus Definition \ref{p-mindef} captures the notion of Markovian graphs \(G\) that are closest to being faithful to \(P\).
\begin{comment}
    \begin{definition}[Causal minimality]
   A distribution  \(P\) is \emph{causally minimal} 
    w.r.t.\ 
    \(G\) if \(P\) is Markovian to \(G\) and $P$ is not Markovian to a  proper subgraph of \(G\).
\end{definition}
\end{comment}
\begin{definition}[Causal minimality \citep{zbMATH01600338}]\label{causalmindef}
   Let $\mathbb{G}$ be a class of graphs. A distribution  \(P\) is \emph{causally minimal} 
    to \ 
    \(G\in \mathbb{G}\) if \(P\) is Markovian to \(G\) and $P$ is not Markovian to any proper subgraph \(G'\in \mathbb{G}\) of \(G\).
\end{definition}
The minimal graphs in Definition \ref{causalmindef} refer to graphs without subgraphs that distribution \(P\) is Markovian to. Again, interpreting DAGs as SEMs, minimal DAGs correspond to SEMs such that there is no SEM obtainable by only removing functional dependence w.r.t. existing arguments, which still generates \(P\). Thus Definition \ref{causalmindef} captures the notion of the simplest DAG which generates \(P\).

%Note that 
The 
notions 
of causal minimality, Pearl-minimality, and sparsest Markov property %may change 
depend 
on 
the given graph class \(\mathbb{G}\), 
whereas 
the notion of minimally Markovian remains the same across different graph classes. 
When \(\mathbb{G}\) is the class of DAGs, Proposition \ref{propforstlit} provides the following chain of implications, from which \(2\Rightarrow 3\Rightarrow 4\) is provided by
\citet{forster}. We show that \(1\Rightarrow 2\) holds, via a proof analogous to that of Proposition \ref{propforst}, where we relate these notions of minimality  for \(\mathbb{G}\) being the class of \emph{maximal}  ancestral graphs (MAGs) 
\citep{10.1214/aos/1031689015}.

\begin{proposition}[Relation of assumptions]
\label{propforstlit}
     For a given distribution \(P\) and the DAG \(G\), the following statements imply the next:
    \begin{enumerate}
        \item \(P\) is minimally Markovian to \(G\).
        \item \(G\) is a sparsest Markov graph of \(P\).
        \item \(P\) is Pearl-minimal to \(G\).
    %w.r.t.\ \(P\).
        \item \(P\) is causally minimal to \(G\).
    \end{enumerate}
\end{proposition}

\begin{definition}[Maximal graphs \citep{10.1214/aos/1031689015}]
    A graph \(G\) is \emph{maximal} if  for any two non-adjacent nodes 
    %\(i\) not adjacent to \(j\) in \(G\), 
    $i,j \in V$,  
    there exists \(C(i,j)\subseteq V\backslash \{i,j\}\) such that \(i\perp_G j\cd C(i,j)\).
\end{definition}

Note that the pairwise Markov property is a property relating a distribution and a graph, whereas the notion of maximality is a notion on a graph itself.

%Some other assumptions in causal learning include the recently introduced \emph{V-OUS and collider-stable} condition by \cite{teh}.  

We also note the recently introduced  \emph{V-ordered upward stability (V-OUS) and collider-stable} condition \citep{teh} for when \(\mathbb{G}\) are DAGs.

\begin{definition}[V-OUS and collider-stable]\label{vous}
    A distribution \(P\) is V-OUS and collider-stable 
    w.r.t.\ 
    a DAG \(G\) if for all v-configurations \(i\sim  k\sim j\) in \(G\), we have:
    \begin{itemize}
        \item \emph{(V-ordered upward stability (V-OUS))}. If \(i\sim  k\sim j\) is a non-collider, then for all \(C\subseteq V\backslash\{i,j,k\}\), we have  
        \[i\ci j\cd C\Rightarrow i\ci j\cd C\cup \ns{k}.\] 
        %\ts{ use backslash ns to enclose elements in sets, when necessary}
        %
        \item \emph{(Collider-stability)}. If \(i\fra k\fla j \), then  there exists \(C\subseteq V\backslash\{i,j,k\}\) such that \(i\ci j\cd C\).
    \end{itemize}
\end{definition}

\begin{comment}
    \begin{figure*}[t]
\centerline{\includegraphics[width=\textwidth,height=9pc,draft]{empty}}
\caption{This is the sample figure caption.\label{fig1}}
\end{figure*}
\end{comment}

\section{Theory and Framework}
\label{theory}

Let \(\mathbb{G}\) denote a class of graphs that is a subclass of anterial graphs. Throughout this section, we will not make any assumptions on \(\mathbb{G}\), unless explicitly stated.

Given a graph \(G\in \mathbb{G}\) and a distribution $P$ on the same set of nodes,  we write  \(\mathcal{A}(P,G)=\top\) if \(P\) satisfies the \emph{property} \(\mathcal{A}\) 
w.r.t.\ 
the  graph \(G\), otherwise we write \(\mathcal{A}(P,G)\neq\top\) if \(P\) does not satisfy property \(\mathcal{A}\) 
w.r.t.\ 
 \(G\). For example,  if \(\mathcal{A}\) is the faithfulness assumption, 
``\(\mathcal{A}(P,G)=\top\)'' can be read as ``\(P\) is faithful to \(G\).'' When \(\mathcal{A}(P,G)=\top\) for every graph \(G\) in set \(\boldsymbol{G}\), we 
%can \
write \(\mathcal{A}(P,\boldsymbol{G})=\top\). 
%without ambiguity.

\begin{remark}
    Alternatively, \(\mathcal{A}\) is a relation relating distributions and graphs, so that $(P, G) \in \mathcal{A}$ if and only if \(\mathcal{A}(P,G)=\top\). However throughout this work, \(\mathcal{A}\) will be referred to as a property, since this terminology is more consistent with the current literature. \erk
    %more  illustrative of what \(\mathcal{A}\) is trying to represent.
\end{remark}

Given a class of graphs \(\mathbb{G}\), and a property \(\mathcal{A}\), we say that \(P\) satisfies \(\mathcal{A}\)-\emph{uniqueness} 
%\(U_{\mathcal{A}}\) 
%if  given a  distribution \(P\), 
if the graphs
$\ns{G\in \mathbb{G}: \mathcal{A}(P,G)=\top}$ belong to the same Markov equivalence class (hence unique up to MEC);
%
%all graphs \(G\in \mathbb{G}\) to which \(P\) satisfies property \(\mathcal{A}\) w.r.t. (all \(G\in \mathbb{G}\) such that \(\mathcal{A}(P,G)=\top\)) are Markov equivalent (unique up to MEC), 
%
if \(\mathcal{A}\)-uniqueness is satisfied,  we write \(U_{\mathcal{A}}(P)=\top\), otherwise \(U_\mathcal{A}(P)\neq\top\). Note the notion of \(\mathcal{A}\)-uniqueness may change depending on the class of graphs \(\mathbb{G}\) 
under consideration. 
%the context of \(\mathbb{G}\) for which it appears in the following will be made clear.

We also give the following definitions.
\begin{definition}[Class property]
    Given a class of graphs \(\mathbb{G}\), a property \(\mathcal{A}\) is said to be a \emph{class property} if for all distributions  $P$ and for all graphs  \(G_1, G_2\in \mathbb{G}\) that are Markov equivalent,  we have
    \[ \mathcal{A}(P,G_1)=\top\iff \mathcal{A}(P,G_2)=\top.\]
\end{definition}

Examples of class properties include many common assumptions such as the sparsest Markov assumption, the 
Pearl-minimality assumption,
and the V-OUS and collider-stable condition. Note that property \(\mathcal{A}\) being a  class property can be seen as a converse to \(\mathcal{A}\)-uniqueness holding for all distributions; for class properties we have Markov equivalence implying the property \(\mathcal{A}\), while for \(\mathcal{A}\)-uniqueness we have the property \(\mathcal{A}\) implying Markov equivalence.

Given an algorithm that outputs a set of graphs \(\boldsymbol{G}(P)\),
from an 
input distribution \(P\), we say that
the property \(\mathcal{A}\) 
\emph{corresponds} 
to the algorithm if
    \[\mathcal{A}(P,G)=\top \text{ for } G\in \mathbb{G} 
    \quad \iff \quad 
    G \in \boldsymbol{G}(P).\] 

For example, when \(\mathcal{A}\) is the sparsest Markov property in Definition \ref{sparsM}, the corresponding algorithm is the SP algorithm from \cite{UhlSP}. Note that we can frame \emph{any} constraint-based causal learning algorithm that outputs a set of graphs \(\boldsymbol{G}(P)\) solely from the input distribution \(P\) without relying on additional specified parameters, as a property \(\mathcal{A}\)---we simply set $\mathcal{A}(P, G) = \top$,  
%for any input distribution \(P\) and 
for all graphs \(G\in \boldsymbol{G}(P)\) and $\mathcal{A}(P, G) \neq \top$, otherwise. Using this description of a constraint-based causal learning algorithm as a property, we present our framework. 

\begin{theorem}[Framework]
\label{mainth}
    %Given an algorithm that outputs the graph \(G(P)\) from a  distribution \(P\) such that \(\mathcal{A}(P,G(P))=\top\), we have
    Given a class of graphs \(\mathbb{G}\), let $\mathcal{A}$ be a property.  Consider a constraint-based causal learning algorithm that outputs a set of graph \(\boldsymbol{G}(P)\subseteq \mathbb{G}\) from a  distribution \(P\), such that \(\mathcal{A}(P,\boldsymbol{G}(P))=\top\). Let  \(G_0\in \mathbb{G}\) be the true causal graph,  then 
    %we have 
    %\begin{itemize}
        %\item [] 
        \[ \mathcal{A}(P,G_0)=\top \text{ and } U_{\mathcal{A}}(P)= \top\quad \Rightarrow \quad  \text{ the algorithm is correct}.\] 
    %\end{itemize}
    In addition, if \(\mathcal{A}\) is a class property, and corresponds to the algorithm,
    then the reverse implication holds.
    
\end{theorem}
 
%For \emph{any} constraint-based causal learning algorithm that outputs a set of graphs \(\boldsymbol{G}(P)\) solely from the input distribution \(P\), without relying on additional specified parameters, we see that our framework includes all such learning algorithms, by simply defining the property \(\mathcal{A}\) as, for any input distribution \(P\), for all graphs \(G\in \boldsymbol{G}(P)\), $\mathcal{A}(P, G) = \top$, and $\mathcal{A}(P, G) \neq \top$ otherwise. 

\begin{remark}
    To see the importance of the class of graphs $\mathbb{G}$ in Theorem \ref{mainth}, recall the definition of \(\mathcal{A}\)-uniqueness---\(P\) satisfies \(\mathcal{A}\)-uniqueness if the graphs \(\{G\in \mathbb{G}: \mathcal{A}(P,G)=\top\}\) are Markov equivalent; this depends on the subclass of anterial graphs \(\mathbb{G}\) being considered. Thus Theorem 1 depends on the subclass of anterial graphs \(\mathbb{G}\) being considered; Theorem 1 holding for a class does not logically imply Theorem 1 holding for a smaller subclass, as will be seen in Remark \ref{classremark}. 
    
    Note that Theorem \ref{mainth} applies to any \(\mathbb{G}\) that is a subclass of anterial graphs, since the notion of Markov equivalence used to define \(\mathcal{A}\)-uniqueness is well defined for anterial graphs. \erk
\end{remark}
If \(\mathcal{A}\) is a class property, then we can  express \(\mathcal{A}\) as a property relating distributions and the \emph{MEC} of graphs, and thus we can substitute
$J(G_0)$ and  $J(\boldsymbol{G}(P))$ for 
$G_0$ and  $\boldsymbol{G}(P)$ in Theorem \ref{mainth}. Similarly, since in constraint-based causal learning,  we are only concerned with \(J(P)\), the conditional independencies of \(P\), we can express Theorem \ref{mainth} as follows.  
%We say that a property $\mathcal{A}$ is \emph{constraint-based} if for all graphs $G\in \mathbb{G}$, we have $A(P, G) = \top \iff A(P', G)=\top$ for all distributions $P$ and $P'$ with  $J(P) = J(P')$.
%%%

\begin{corollary}[Theorem \ref{mainth} in terms of MEC] 
\label{maincor}
For a class property \(\mathcal{A}\) that corresponds to the algorithm, we have
\[\mathcal{A}(J(P),J(G_0))=\top \text{ and } U_{\mathcal{A}}(J(P))=\top  \iff  \text{ the algorithm is correct.}\]
\end{corollary}
%If $G(P)$ is the output of an \emph{arbitrary} constraint-based causal learning algorithm that outputs a graph that solely depends on input distribution \(P\), without relying on additional mathematical structures to pick among the set of graphs \(G\) such that \(\mathcal{A}(P,G)=\top\), to output, such as a depth parameter in depth-first search, we see that our framework includes all such learning algorithms, by simply setting $\mathcal{A}(P, G(P)) = \top$ for all distributions $P$. 
%
% 

Theorem \ref{mainth} states that given any causal learning algorithm, if we can identify the corresponding property \(\mathcal{A}\) that relates the output set \(\boldsymbol{G}(P)\) and input distribution \(P\), using the same property \(\mathcal{A}\), by taking its \emph{conjunction} with  \(U_{\mathcal{A}}\), we can then immediately obtain correctness conditions for the corresponding algorithm. This will be illustrated further in Example \ref{smr} and in Section \ref{type} with the PC algorithm. 

Since any constraint-based causal learning algorithm can be expressed as a corresponding property and, in Theorem \ref{mainth}, nothing is assumed about the property \(\mathcal{A}\), Theorem \ref{mainth} allows for a systematic study of the correctness condition of \emph{any} causal learning algorithm. This is done by considering all possible properties \(\mathcal{A}\) and the correctness condition of its corresponding algorithm, \(\mathcal{A}\) and \(U_{\mathcal{A}}\), as will be seen in Section \ref{comp} and Section \ref{degen}.

%A property \(\mathcal{A}\) %is generic, \(\mathcal{A}\) 
%can be chosen to  describe how the output graph \(G(P)\) is obtained from the input distribution \(P\) of an \emph{arbitrary } constraint-based causal learning algorithm that outputs a graph \(G(P)\) solely from the input distribution \(P\) without relying on additional mathematical structures such as a depth parameter in depth-first search.

%\ts{it is not clear what this sentence is trying to say, and it just reads as the property A appears in many places in thm1} 

In Theorem \ref{mainth}, the property \(\mathcal{A}\) describes both the output of the corresponding algorithm (using the term \(\mathcal{A}(P,\boldsymbol{G}(P))=\top\)) \emph{and} the algorithm's correctness condition (using the terms \(\mathcal{A}(P,G_0)=\top\) and \(U_{\mathcal{A}}(P)=\top\)). This duality suggests a different paradigm for designing causal learning algorithms. Instead of first designing the computational steps of the algorithm and then proving that the algorithm works correctly under standard correctness conditions such as faithfulness; the designer first envisions how the output of the algorithm should be related to the input distribution (via a corresponding property \(\mathcal{A}\)), and by studying the correctness condition (\(\mathcal{A}\) and \(U_{\mathcal{A}}\)), the designer can ensure that the correctness condition of the designed algorithm is not too strong, \emph{before} designing the actual steps of the corresponding algorithm that outputs \(\boldsymbol{G}(P)\) which satisfies \(\mathcal{A}(P,\boldsymbol{G}(P))=\top\), on which Theorem \ref{mainth} contains no information. We summarise this paradigm succinctly as follows: 

\begin{center}
     \begin{tikzpicture}[>=stealth]
     \node at (0,0) {
    \begin{minipage}{\linewidth}
      \begin{enumerate}
          \item Select property \(\mathcal{A}\).
          \item Study the correctness condition \(\mathcal{A}\) and \(U_{\mathcal{A}}\).
          \item  Design actual steps of the algorithm corresponding to \(\mathcal{A}\) \quad  (Not informed by Theorem \ref{mainth}).
      \end{enumerate}
    \end{minipage}
  };
 \draw[decorate,decoration={brace,amplitude=10pt,mirror}] (-2,-0.2) -- (-2,0.65)
    node[midway, xshift=2.5cm] {Informed by Theorem \ref{mainth}};
    %\draw[decorate,decoration={brace,amplitude=10pt,mirror}] (-3.5,-0.6) -- (-3.5,-0.2) node[midway, xshift=3cm] {Not Informed by Theorem \ref{mainth}};
    \end{tikzpicture}
\end{center}

In the following examples, we illustrate how we can recover correctness conditions for some existing causal learning algorithms using Theorem \ref{mainth} with different properties. The examples also allude to the suggested design paradigm.

\begin{example}
%[\(\mathcal{A}\) being the sparsest Markov assumption]
[Sparsest Markov property]
\label{smr}Let the class  of graphs \(\mathbb{G}\) be DAGs. The sparsest permutation (SP) algorithm \citep{UhlSP} efficiently builds causally minimal DAGs from permutations, and from these DAGs outputs \(\boldsymbol{G}(P)\),  the set of DAGs with the least number of edges. Thus, it can be seen that the corresponding property \(\mathcal{A}\) of the SP algorithm is the sparsest Markov property, since \(\mathcal{A}(P, G)=\top\) if and only if \(G\in \boldsymbol{G}(P)\); i.e. \(G\) is a sparsest Markov graph of \(P\). Since the property \(\mathcal{A}\) is a class property and corresponds to the SP algorithm, Theorem \ref{mainth} states that 
the exact correctness condition of the SP algorithm is \(\mathcal{A}(P,G_0)=\top\) and \(U_{\mathcal{A}}(P)=\top\), which is equivalent to \(P\) satisfying the sparsest Markov representation (SMR) assumption w.r.t.\ true causal graph \(G_0\). Note that the SMR assumption has also been formulated as unique-frugality 
    %in literature such as 
    \citep{lam}. 
    
Without knowing the actual steps of the SP algorithm and just by considering the property \(\mathcal{A}\) which is the sparsest Markov property, Theorem \ref{mainth} states that the exact correctness condition of the algorithm corresponding to the property \(\mathcal{A}\) is the SMR assumption. Attempting to construct the actual corresponding algorithm that outputs all of the sparsest Markov graphs of input distribution \(P\) would then suggest the SP algorithm.  \erk
\end{example}

\begin{example}
%[\(P\) is adjacency faithful and V-OUS and collider-stable w.r.t. \(G\)]
[Adjacency faithfulness and stability]
    Let the class of graphs \(\mathbb{G}\) be DAGs. By modifying the orientation rules of the PC algorithm, the output set $\boldsymbol{G}(P)$  of the ``Me-LoNS'' algorithm is constructed as the set of DAGs such that input distribution \(P\) is adjacency faithful and V-OUS and collider-stable w.r.t. \(G\in \boldsymbol{G}(P)\) \citep{teh}. Thus, the corresponding property \(\mathcal{A}\) of the Me-LoNS algorithm is adjacency faithful and V-OUS and collider-stability, since \(\mathcal{A}(P,G)=\top\) if and only if \(G\in \boldsymbol{G}(P)\) is adjacency faithful and V-OUS and collider-stable 
    w.r.t.\
    \(G\in \mathbb{G}\). Since \(\mathcal{A}\) is a class property, Theorem \ref{mainth} then states that \(\mathcal{A}(P,G_0)=\top\) and
    \(U_{\mathcal{A}}(P)=\top\) is the exact correctness condition for the Me-LoNS algorithm. Note that \(U_{\mathcal{A}}(P)=\top\) is equivalent to the notion of \(P\) being \emph{modified V-stable}, which is defined algorithmically in \citet[Definition 11]{teh}.

   Without knowing the actual steps of the Me-LoNS algorithm and just by considering the property \(\mathcal{A}\) which is adjacency faithful and V-OUS and collider-stability, Theorem \ref{mainth} states that the exact correctness condition of the corresponding algorithm to property \(\mathcal{A}\) are the conditions in \citet{teh} which are strictly weaker than faithfulness. Attempting to construct the actual corresponding algorithm that outputs all of the graphs \(G\) that the input distribution \(P\) is adjacency faithful and V-OUS and collider-stable w.r.t.\ would then suggest the Me-LoNS algorithm. \erk
\end{example}

\begin{remark}
     Correctness conditions for constraint-based causal learning algorithms, such as faithfulness, are untestable since we do not have access to the true causal graph \(G_0\). As such, it is uncertain whether these conditions hold in real-world domains (such as inferring gene regulatory networks), and interpretations from subject matter experts may be needed to assert their validity. This issue can be addressed using our paradigm by selecting a property \(\mathcal{A}\) that is interpretable by subject matter experts, this guarantees that the corresponding algorithm has a correctness condition (\(\mathcal{A}\) and \(U_{\mathcal{A}}\)) that is also interpretable and verifiable by subject matter experts. An example of such properties could be the minimality assumptions in Section \ref{assume}, which all implies that the true causal graph is the `simplest' graph that generates the observational distribution. In Section \ref{degen}, we will apply this paradigm to study these different minimality assumptions.
     
     The subject specific interpretation of what the selected property should be is outside the scope of our work. \erk
\end{remark}

\bmsubsection*{Including background knowledge}

Constraint-based causal learning algorithms often involve the inclusion of background knowledge such as the presence or absence of certain edges in the true causal graph \citep{meek}, which can be described as a property \(\mathcal{E}\) as follows:
\begin{align*}
    \mathcal{E}(P,G)=\top \text{ for } G\in \mathbb{G} \quad \iff \quad G \text{ satisfies the constraints imposed by the background knowledge }
\end{align*}

    %In Section \ref{pminimp}, we deduced that any meaningful property \(\mathcal{A}\) for causal learning should at least satisfy Pearl-minimality. Since the forbidden and required edges contained in the background knowledge are allowed to be arbitrary, the corresponding property \(\mathcal{E}\) need not satisfy Pearl-minimality, as such background knowledge will be treated and denoted as a separate property \(\mathcal{E}\) throughout.

%In Section \ref{pminimp}, by considering what should be required from the output \(\boldsymbol{G}(P)\) of the algorithm that corresponds to a meaningful property \(\mathcal{A}\) for causal learning, we deduced that any such \(\mathcal{A}\) should at least satisfy Pearl-minimality. Note that by the duality in Theorem \ref{mainth}, property \(\mathcal{A}\) also relates the input distribution \(P\) to the true causal graph \(G_0\), thus \(P\) is also Pearl-minimal to \(G_0\).

In \cite{meek}, background knowledge is of the form of required and forbidden edges, thus the definition of the corresponding property \(\mathcal{E}\) does not involve the distribution \(P\). However, we will allow more elaborate forms of background knowledge which do not depend on the distribution \(P\), such as in Example \ref{degenex} when all colliders are absent in the graph. %For more elaborate examples of background knowledge, such as those in Table \ref{tab:example}, the corresponding property \(\mathcal{E}\) then depends on the distribution \(P\).

Given an algorithm with corresponding property \(\mathcal{A}\) and background knowledge \(\mathcal{E}\), we can obtain the corresponding property of the algorithm after including background knowledge by taking the conjunction of properties  $\mathcal{A} \wedge \mathcal{E}$, since the output graphs \(G\in \boldsymbol{G}(P)\) of the corresponding algorithm have to satisfy both $\mathcal{A}(P,G) = \top$ and $\mathcal{E}(P,G) = \top$.
%we see that depending on \(\mathcal{A}\)-uniqueness, %\(U_{\mathcal{A}}\), relaxed correctness conditions on the true causal graph \(G_0\) correspond to a relaxed algorithm outputting \(\boldsymbol{G}(P)\).
\subsection{Comparing correctness Conditions}\label{comp}

Recall that in Theorem \ref{mainth}, \(\mathcal{A}(P,G_0)=\top\) and \(U_{\mathcal{A}}(P)=\top\) is the correctness condition of the algorithm that the property \(\mathcal{A}\) corresponds to, which returns the set of graphs \(\boldsymbol{G}(P)\) from the input distribution \(P\) such that \(\mathcal{A}(P,\boldsymbol{G}(P))=\top\). Since any constraint-based causal learning algorithm can be expressed as a corresponding property and nothing is assumed about the property \(\mathcal{A}\), we can describe the correctness condition of \emph{any} causal learning algorithm as  \(\mathcal{A}\) and \(U_{\mathcal{A}}\), for some corresponding property \(\mathcal{A}\).

A common critique of constraint-based causal learning is that the conditions under which algorithms are correct, such as the faithfulness assumption, can be too strong.  
In the literature, many attempts have been made to weaken correctness conditions for causal learning \citep{UhlSP, UhlGSP, lam}. Theorem \ref{mainth} suggests the design paradigm in which the designer first selects a property and studies how strong the correctness condition of the corresponding algorithm is, before finally designing the actual steps of the corresponding algorithm. Thus, given two properties \(\mathcal{A}\) and \(\mathcal{B}\), it will be of interest to be able to compare the resulting correctness conditions for the corresponding algorithms, which will be addressed here. The designer can then select the property which results in a corresponding algorithm that is correct under weaker correctness conditions. 

Given any property \(\mathcal{A}\), to compare the strength of the resulting correctness condition \(\mathcal{A}(P,G_0)=\top\) and \(U_{\mathcal{A}}(P)=\top\) of the corresponding algorithm. We introduce the notion of a support of \(\mathcal{A}\), which is the set of distributions and graphs on which the resulting correctness condition of the corresponding algorithm (\(\mathcal{A}\) and \(U_{\mathcal{A}}\)) holds.

%Based on the formulation in Theorem \ref{mainth}, we can compare the strength of causal learning consistency conditions as follows. 
%
\begin{definition}[Support]\label{support}
    Given a class of graphs \(\mathbb{G}\), and a property \(\mathcal{A}\) relating distributions and graphs in \(\mathbb{G}\), the 
    \emph{support} of \(\mathcal{A}\)
    %, 
    %\(\supp(\mathcal{A})\) is the class
    is given by the class of tuples given by
    \begin{equation*}
        \supp(\mathcal{A})=\ns{(P,G): G\in \mathbb{G}, \  \mathcal{A}(P,G)=\top \text{  and  } U_{\mathcal{A}}(P)=\top}.
    \end{equation*}
\end{definition}

\begin{example}%[$\mathcal{A}$ being faithfulness]
[Faithfulness]
\label{A1}
    Given any class of graphs \(\mathbb{G}\), let \(\mathcal{A}(P,G)=\top\) if \(P\) is faithful 
    w.r.t.\
    \(G\in \mathbb{G}\). 
    Since \(U_{\mathcal{A}}(P)=\top\) if \(\mathcal{A}(P,G_P)=\top\) for some \(G_P\), taking the conjunction with \(U_{\mathcal{A}}\) is trivial in this case; \(\mathcal{A}(P,G)=\top\) and \(U_{\mathcal{A}}(P)=\top\) is still equivalent to saying \(P\) is faithful 
    w.r.t.\
    \(G\). When property \(\mathcal{A}\) is faithfulness, we have \(\text{supp}(\mathcal{A})\) being all pairs \((P,G)\) such that \(G\in \mathbb{G}\) and distributions \(P\)  faithful to \(\mathcal{G}\). \erk
\end{example}
Given two properties $\mathcal{A}$ and  $\mathcal{B}$ relating distributions and graphs, the resulting correctness condition of the algorithm that \(\mathcal{A}\) corresponds to is \emph{stronger} than the correctness condition of the algorithm that \(\mathcal{B}\) corresponds to, if 
\begin{align*}
    \mathcal{A}(P,G_0)=\top \text{ and } U_{\mathcal{A}}(P)=\top \Rightarrow \mathcal{B}(P,G_0)=\top \text{ and } U_{\mathcal{B}}(P)=\top.
\end{align*}
Indeed, logically, the stronger correctness condition implies the weaker condition.

Since \(G_0\in \mathbb{G}\) is the unknown true causal graph and \(P\) can be any input distribution, the implication should hold for all distributions \(P\) and graphs \(G\in \mathbb{G}\). Thus, using the support, the assertion that the corresponding algorithm of property \(\mathcal{A}\) has a stronger correctness condition than the corresponding algorithm of property \(\mathcal{B}\), can be simply expressed as \(\supp(\mathcal{A})\subseteq \supp(\mathcal{B})\).

%then we say that \(\mathcal{A}\) and \(U_{\mathcal{A}}\) is \emph{stronger} than \(\mathcal{B}\) and \(U_{\mathcal{B}}\).

From Examples \ref{A1}, \ref{smr}, and \ref{A.ex}, we see that in general given properties \(\mathcal{A}\) and $\mathcal{B}$ relating distributions and graphs, the assertion that the property \(\mathcal{A}\) by \emph{itself} is stronger than the  property \(B\), i.e. \(\mathcal{A}(P,G)=\top\Rightarrow \mathcal{B}(P,G)=\top\) for all \(P\) and $G\in \mathbb{G}$, does not necessarily imply \(\supp(\mathcal{A})\subseteq \supp(\mathcal{B})\) nor \(\supp(\mathcal{B})\subseteq \supp(\mathcal{A})\). However, under some conditions, we can relate \(\supp(\mathcal{A})\) and \(\supp(\mathcal{B})\), thus comparing the strength of the correctness condition of the algorithms corresponding to properties \(\mathcal{A}\) and \(\mathcal{B}\) respectively.

\begin{proposition}[Reversing the implication]
\label{reverse}
    Let  \(\mathcal{A}\) and \(\mathcal{B}\) be properties relating distributions and graphs in \(\mathbb{G}\).  Suppose that  for all \(P\), we have
    \begin{equation*}
        \exists G_P\in \mathbb{G} \text{ such that } \mathcal{A}(P,G_P)=\top\quad \iff \quad \exists G'_P\in \mathbb{G} \text{ such that } \mathcal{B}(P,G'_P)=\top
    \end{equation*}
    and \(\mathcal{A}\) is a class property. If \(\mathcal{A}(P,G)=\top\Rightarrow \mathcal{B}(P,G)=\top\) for all \(P\) and \(G\in \mathbb{G}\), then \textnormal{supp}\((\mathcal{B})\subseteq\) \textnormal{supp}\((\mathcal{A})\).
\end{proposition}
As an example, the conditions in Proposition \ref{reverse} apply when the algorithms corresponding to the properties \(\mathcal{A}\) and \(\mathcal{B}\) always return a graph given any input distribution. This can be seen with the SP algorithm \citep{UhlSP} which corresponds to the Sparsest Markov property, and always outputs a graph given any input distribution; this will be seen in Section \ref{mincompare}.
\begin{proposition}[Preserving the implication]
\label{preserve}
     Suppose that for properties \(\mathcal{A}\) and \(\mathcal{B}\) relating distributions and graphs in \(\mathbb{G}\), for all \(P \text{ such that } \mathcal{A}(P,G_P)=\top \text{ for some } G_P\in \mathbb{G}\), we have the equality
     \begin{equation*}
         \{G \in \mathbb{G}: \mathcal{A}(P,G)=\top\}=\{G \in \mathbb{G}:  \mathcal{B}(P,G)=\top\}.
     \end{equation*}
    If \(\mathcal{A}(P,G)=\top\Rightarrow \mathcal{B}(P,G)=\top\) for all \(P\) and \(G\in \mathbb{G}\), then \textnormal{supp}\((\mathcal{A})\subseteq\) \textnormal{supp}\((\mathcal{B})\).
\end{proposition}
Propositions \ref{reverse} and \ref{preserve} try to address the question: ``If, by themselves, property \(\mathcal{A}\) is stronger than property \(B\), then 
under what conditions on \(\mathcal{A}\) and \(\mathcal{B}\) can one deduce which of the corresponding algorithms has a stronger correctness conditions?'' These propositions will be used in Section \ref{mincompare} to compare correctness conditions of algorithms that correspond to the minimality notions from Section \ref{assume}.

Note that after including background knowledge, the property \(\mathcal{A}\wedge \mathcal{E}\) is always stronger than the property \(\mathcal{A}\). We can then apply Proposition \ref{reverse} to provide conditions when including background knowledge \(\mathcal{E}\) weakens the correctness conditions of the corresponding algorithm of property \(\mathcal{A}\); i.e. \textnormal{supp}\((\mathcal{A})\subseteq \textnormal{supp}(\mathcal{A}\wedge\mathcal{E})\). This is summarised as the following corollary.

\begin{corollary}[Weakening correctness conditions using background knowledge] 
\label{weaktoGP}
Let \(\mathcal{A}\) be a class property.
Let  \(\mathcal{E}\) be another property.  If   for all \(P\), we have
\begin{align*}
        \exists G_P\in \{G: \mathcal{A}(P,G)=\top\} \text{ such that } \mathcal{E}(P,G_P)=\top,
\end{align*}
then \textnormal{supp}\((\mathcal{A})\subseteq \textnormal{supp}(\mathcal{A}\wedge\mathcal{E})\).
\end{corollary}

%This corollary will be applied in Section \ref{degen} to relax and weaken the faithfulness assumption.

%\ref{degen}, Corollary \ref{weaktoGP} will be applied in the case when \(\mathcal{A}\) is Pearl-minimality, to relax the faithfulness assumption.
%
%

\section{Applications to DAGs}
\label{type}
Recall that 
given a constraint-based causal learning algorithm, if we can identify the corresponding property that relates the output set of graphs and the input distribution, Theorem \ref{mainth} allows us to obtain the correctness condition of the algorithm. In this section, restricting our attention to \(\mathbb{G}\) being the class of  DAGs, and focusing on causal learning algorithms that use an orientation rule to identify colliders and non-colliders in a skeleton such as the PC algorithm and ``Me-LoNS'' \citep{teh}, we will frame these orientation rules as local constructed properties to identify the corresponding property of the algorithm. In particular, we find the exact correctness conditions for the PC algorithm.

%In this section, we  will study properties that will be used to identify colliders and non-colliders.  We will restrict our attention to \(\mathbb{G}\) being the class of  DAGs, and find exact consistency conditions for the PC algorithm.   %\ts{The PC algo gives you a DAG?, is this going to cause problems in the writing?}
%We omit the the presence of \(\mathbb{G}\).
%to the  the class of graphs \(\mathbb{G}\) to be DAGs, and we will omit the the presence of \(\mathbb{G}\). %will not be made explicit to avoid clutter.

\subsection{Local  Properties}

Consider an algorithm that takes in an input distribution $P$ and constructs DAGs locally by mapping v-configurations into colliders and non-colliders via orientation rules, we will describe how $P$ relates to the v-configurations of the output graphs locally, by defining local constructed properties for DAGs.

%The statement \(U_{\mathcal{A}}(P)=\top\) tells us information about the property \(\mathcal{A}\) via the uniqueness of the Markov equivalence class; this information, together with results characterising Markov equivalence from \citet{pearl} allows us to  define locally constructed properties for DAGs.
%the class of graphs \(\mathbb{G}\) being

As in Section \ref{theory},
we write  \(\mathfrak{v}(P,i\sim k\sim j)=\top\) if given a distribution \(P\) and a v-configuration \(i\sim k \sim j\) in \(\sk(P)\), the distribution \(P\) satisfies the property \(\mathfrak{v}\) w.r.t. \(i\sim k \sim j\); we refer to such a $\mathfrak{v}$ as a \emph{local property}. Specifically, when the v-configuration is a non-collider/collider, we will denote the local property as $\mathfrak{n}$/$\mathfrak{c}$ respectively. From this, we define the following class of properties relating distributions and graphs:

%Given two local properties $\mathfrak{n}$ and $\mathfrak{c}$, relating distributions and v-configurations in \(\sk(P)\), we define the following class of properties relating distributions and graphs:
\begin{definition}[\(\mathcal{V}_{\mathfrak{n,c}}\)]
\label{def-vnc}
    Given properties \(\mathfrak{n}\) and \(\mathfrak{c}\) relating distributions and v-configurations in \(\sk(P)\), we write \(\mathcal{V}_{\mathfrak{n,c}}(P,G)=\top\) if:
    \begin{enumerate}[1]
        \item
        \label{item1-vnc}
        \(\sk(P)=\sk(G)\).
        \item For every v-configuration \(i\sim k \sim j\) in \(\sk(P)\), we have
        \begin{itemize}
            \item
            \label{item2-vnc}
            %\item[]
            \(i\sim k \sim j\) is a collider in \(G \Rightarrow \mathfrak{c}(P,i\sim k\sim j)=\top\).
            \item
            \(i\sim k \sim j\) is a non-collider in \(G \Rightarrow \mathfrak{n}(P,i\sim k\sim j)=\top\).
        \end{itemize}
    \end{enumerate}
\end{definition}
% \begin{remark}
    %If local properties \(\mathfrak{n}\) and \(\mathfrak{c}\) imply that \(P\) is Markovian to \(G\) via Item 2, then Item 1 implies minimally Markovian, and consequently so does \(\mathcal{V}_{\mathfrak{n,c}}\), thus by Proposition \ref{propforstlit}, it can be seen that Pearl-minimality is implied by \(\mathcal{V}_{\mathfrak{n,c}}\). \erk
%\end{remark}
%
%

%If we have some conditions on property \(\mathfrak{n}\) and \(\mathfrak{c}\), we have that \(\mathcal{V}_{\mathfrak{n,c}}-\)uniqueness holds:
%
The property
\(\mathcal{V}_{\mathfrak{n,c}}\) can be thought of as a global property relating $P$ to the whole graph $G$ when each non-collider/collider satisfies the  local property $\mathfrak{n}/\mathfrak{c}$, respectively. This happens when an algorithm outputs graphs using orientation rules constructed from 
$\mathfrak{n}$ and $\mathfrak{c}$ to assign v-configurations (see (\(\mathfrak{n,c})\)-orientation rule in Definition \ref{ncorientrule} ).

With results characterising Markov equivalence from \citet{pearl}, it is straightforward to see that \(\mathcal{V}_{\mathfrak{n,c}}\) is a class property. If the properties \(\mathfrak{n}\) and \(\mathfrak{c}\) do not simultaneously hold, then we have \(\mathcal{V}_{\mathfrak{n,c}}\)-uniqueness.

\begin{proposition}[\(\mathcal{V}_{\mathfrak{n,c}}\) is a class property and \(\mathcal{V}_{\mathfrak{n,c}}\)-uniqueness]
\label{Auniq}
Let \(\mathfrak{n}\) and \(\mathfrak{c}\) be local  properties.  
\begin{enumerate}
    \item \(\mathcal{V}_{\mathfrak{n,c}}\) is a class property.
    \item If \(\text{for all } P\) and all v-configurations \(i\sim k\sim j\) in \(\sk(P)\), we do \textbf{not} have
    \begin{equation}
    \label{double-T}
        \mathfrak{n}(P,i\sim k\sim j)=\top \text{ and } \mathfrak{c}(P,i\sim k\sim j)=\top,
    \end{equation} then  \(U_{\mathcal{V}_{\mathfrak{n,c}}}(P)=\top\) for all \(P\) such that \(\mathcal{V}_{\mathfrak{n,c}}(P,G_P)=\top\) for some \(G_P\).
\end{enumerate}
\end{proposition}

From local properties \(\mathfrak{n}\) and \(\mathfrak{c}\), we  define the \(\mathfrak{(n,c)}\)-\emph{orientation rule} which encapsulates many orientations rules used in constraint-based learning to assign colliders and non-colliders. 
\begin{definition}[\(\mathfrak{(n,c)}\)-orientation rule w.r.t. \(P\)]\label{ncorientrule}
    Given a distribution $P$, for all v-configurations \(i\sim k\sim j\) in \(\sk(P)\), the \(\mathfrak{(n,c)}\)-\emph{orientation rule} w.r.t.\ 
    \(P\) assigns:
    \begin{enumerate}
        \item \(i\sim k \sim j\) as a collider if \(\mathfrak{n}(P,i\sim k\sim j)\neq \top\).
        \item \(i\sim k \sim j\) as a non-collider if \(\mathfrak{c}(P,i\sim k\sim j)\neq \top\).
        \item \(i\sim k \sim j\) as unassigned otherwise.
    \end{enumerate}
\end{definition}

Under some natural conditions on the properties \(\mathfrak{n}\) and \(\mathfrak{c}\), the corresponding \(\mathfrak{(n,c)}\)-orientation  rule 
w.r.t.\ 
\(P\) is both well-defined, so that no v-configurations are simultaneously assigned as a collider and non-collider,  and characterises the set of all graphs such that property \(\mathcal{V}_{\mathfrak{n,c}}\) holds w.r.t. \(P\).
\begin{proposition}[The \(\mathfrak{(n,c)}\)-orientation rule w.r.t. \(P\) is well-defined and is a characterisation of \(\mathcal{V}_{\mathfrak{n,c}}\)]
\label{Achar}
If for all distributions $P$ and all $i\sim k\sim j$ in  $\sk(P)$, we have
    \begin{equation*}
        \mathfrak{n}(P,i\sim k\sim j)=\top \textnormal{ or } \mathfrak{c}(P,i\sim k\sim j)=\top,
    \end{equation*} then:
\hspace{2cm}
    \begin{enumerate}
        \item The \(\mathfrak{(n,c)}\)-orientation rule w.r.t. \(P\) is well-defined.
        \item  Given a distribution \(P\) and a graph \(G\), such that \(\sk(P)=\sk(G)\), the following are equivalent:
    \begin{enumerate}
        \item 
        If \(i\sim k\sim j \text{ is assigned to be a collider/non-collider} \Rightarrow i\sim k\sim j \text{ is a collider/non-collider in } G\).
        \item \(\mathcal{V}_{\mathfrak{n,c}}(P,G)=\top\).
    \end{enumerate}
    \end{enumerate}
\end{proposition}
Given the algorithm using the \((\mathfrak{n,c})\)-orientation rule.
Item 2a in Proposition \ref{Achar} describes the output set \(\boldsymbol{G}(P)\) of the algorithm, thus Item 2 in Proposition \ref{Achar} states that property \(\mathcal{V}_{\mathfrak{n,c}}\) corresponds to the algorithm.

Given \(\sk(P)=\sk(G)\), we can set up a well-defined orientation rule from an algorithm as an \((\mathfrak{n,c})\)-orientation rule using Definition \ref{ncorientrule}, by letting  
%\(\mathfrak{n,c}\) 
each of 
 $\mathfrak{n}$ and $\mathfrak{c}$ 
be the negation of the rule used to assign colliders/non-colliders respectively. If the resulting local properties 
%\(\mathfrak{n,c}\) 
$\mathfrak{n}$ and $\mathfrak{c}$
satisfy the condition in Proposition \ref{Achar}, then using Proposition \ref{Achar}, we have that the resulting 
property
\(\mathcal{V}_{\mathfrak{n,c}}\) corresponds to the algorithm using the orientation rule. In Section \ref{PCcondssec}, this will be applied to the orientation rule of the PC algorithm to obtain the corresponding property, which is then used to derive exact correctness conditions via Theorem \ref{mainth}.

\begin{remark}\label{cav1}
    %For Item 2(a) in Proposition \ref{Achar}, note that only one direction of implication is necessary. \ts{what direction?}
   The  \((\mathfrak{n,c})\)-orientation rule 
   w.r.t.\ 
   \(P\) is defined to assign colliders when \(\mathfrak{n}\) (the local property for non-colliders) does \emph{not} hold instead of when \(\mathfrak{c}\) (the local property for colliders) holds; likewise when assigning non-colliders. This is to capture \emph{all} graphs \(G\) such that \(\mathcal{V}_{\mathfrak{n,c}}(P,G)=\top\); under the condition in Proposition \ref{Achar},  v-configurations \(i\sim k\sim j\) in sk\((G)\) that are: 
    \begin{enumerate}
        \item assigned as colliders/non-colliders cannot be non-colliders/colliders in \(G\), otherwise this will contradict property \(\mathcal{V}_{\mathfrak{n,c}}\);
        \item 
        unassigned can satisfy both \(\mathfrak{n}(P,i\sim k\sim j)=\top\) and \(\mathfrak{c}(P,i\sim k\sim j)=\top\), simultaneously.
    \end{enumerate}
    Otherwise, the output of the orientation rule may not capture all such graphs \(G\) such that \(\mathcal{V}_{\mathfrak{n,c}}(P,G)=\top\), since some of the assigned colliders could potentially be non-colliders as well. This allows us to establish the correspondence of \(\mathcal{V}_{\mathfrak{n,c}}\) to the algorithm which uses the \((\mathfrak{n,c})\)-orientation rule. % to appeal to the converse of Theorem \ref{mainth} for exact consistency conditions.
    \erk
\end{remark}
%

%the corresponding property since the rule is defined via the negated versions of \(\mathfrak{n}\) and \(\mathfrak{c}\).  

\begin{example}
    We substitute \(\mathfrak{n}\) and \(\mathfrak{c}\) as follows:
\begin{itemize}
    \item
    \(\mathfrak{n}(P,i\sim k\sim j)=\top\) if \(\forall C \subseteq \vijk\), 
    %such that 
    %\(i,j,k\not \in C\),
    we have \(i\ci j\cd C\Rightarrow i\ci j\cd C\cup \ns{k}\).
    \item
    \(\mathfrak{c}(P,i\sim k\sim j)=\top\) if \(\exists C \subseteq \vijk\) 
    such that  
    %\(i,j,k\not \in C\) and 
    \(i\ci j\cd C\).
\end{itemize}
Then \(\mathcal{V}_{\mathfrak{n,c}}\) is the V-OUS and collider-stable condition, and the \(\mathfrak{(n,c)}\)-orientation rule is the V-OUS and collider-stable orientation rule in \citet[Definition 10]{teh}. 
\erk
\end{example}

\subsection{Exact correctness Conditions for the PC algorithm}\label{PCcondssec}
Depending on the computational implementation of the orientation rules of the PC algorithm \citep{doi:10.1177/089443939100900106}, which all give the same output under the faithfulness assumption, we can obtain necessary and sufficient conditions for the PC algorithm using Theorem \ref{mainth}.  The  recent \verb|causal-learn| package  in Python \citep{python} contains such implementations.

Throughout this subsection, we let the set of output graphs of the PC algorithm, denoted as \(\boldsymbol{G}(P)\), be the set of DAGs in the Markov equivalence class represented by the CP-DAG output. We also assume that adjacency faithfulness holds, \(\sk(G)=\sk(P)\) for every graph \(G\in \boldsymbol{G}(P)\).  As such, the property \(\mathcal{V}\) that corresponds to the PC algorithm 
%via \(\mathcal{V}(P,G(P))=\top\) 
has to at least satisfy \(\sk(P)= \sk(G)\) for given \(P\) and \(G\); i.e. if \(\mathcal{V}(P,G)=\top\), then \(\sk(P)= \sk(G)\).

\begin{proof}[PC orientation rules]
\phantom\qedhere
\end{proof}

Depending on the computational implementation of the  PC algorithm, one of the following orientation rules 
is employed for  all v-configurations \(i\sim k \sim j\) in \(\sk(P)\). 
%

%
%
\begin{comment}
\begin{enumerate}
        \item If \(\forall\) \(C \subseteq \vij\) s.t.\ \(i\ci j\cd C\), we have \(k\not \in C\), then  assign \(i\sim k\sim j\) as a collider; otherwise assign as a non-collider.
        \item If \(\forall\) \(C \subseteq \vij\) s.t.\ \(i\ci j\cd C\), we have \(k \in C\), then assign \(i\sim k\sim j\) as a non-collider; otherwise assign as a collider.
        \item 
        \begin{enumerate}[(a)]
            \item 
        If \(\exists\) \(C \subseteq \vij\) s.t.\ \(i\ci j\cd C\) \text{ and } \(k \in C\), then assign \(i\sim k\sim j\) as a non-collider.
        \item 
        If  \(\exists\) \(C \subseteq \vij\) s.t.\ \(i\ci j\cd C\) and  \(k \not \in C\), then assign \(i\sim k\sim j\) as a collider. 
        \end{enumerate}
\end{enumerate} 
\end{comment}
%
%

\begin{enumerate}

        \item If \(\forall\) \(C \subseteq \vij\) s.t.\ \(i\ci j\cd C\), we have \(k\not \in C\), then  assign \(i\sim k\sim j\) as a collider; otherwise assign as a non-collider.
        \item If \(\forall\) \(C \subseteq \vij\) s.t.\ \(i\ci j\cd C\), we have \(k \in C\), then assign \(i\sim k\sim j\) as a non-collider; otherwise assign as a collider.
        \item          \begin{enumerate}[(a)]
        \item 
        If  \(\forall\) \(C \subseteq \vij\) s.t.\ \(i\ci j\cd C\), we have  \(k \not \in C\), then assign \(i\sim k\sim j\) as a collider. 
            \item 
        If \(\forall\) \(C \subseteq \vij\) s.t.\ \(i\ci j\cd C\), we have \(k \in C\), then assign \(i\sim k\sim j\) as a non-collider.
        \item 
        Otherwise, leave \(i\sim k\sim j\) is unassigned. 
        \end{enumerate} 
\end{enumerate} 

\begin{proof}[Corresponding PC properties]
\phantom\qedhere
\end{proof}

Note that each orientation rule is an example of an  \(\mathfrak{(n,c)}\)-orientation rule 
w.r.t.\ 
the
input distribution \(P\), and we can convert this orientation rule into local properties \(\mathfrak{n}\) and \(\mathfrak{c}\), and thus the corresponding property \(\mathcal{V}_{\mathfrak{n,c}}\) from negating the orientation rule used to assign colliders and non-colliders respectively (Proposition \ref{Achar}). For each of the orientation rule \(I\) which defines an algorithm, the corresponding property \(\mathcal{V}_{\mathfrak{n,c}}\) is denoted as \(\mathcal{V}_I\) as follows.   For \(P\) and \(G\), and v-configurations \(i\sim k\sim j\) in \(G\), we set:
\begin{enumerate}
    \item \(\mathcal{V}_1(P,G)=\top\), 
    \quad
    if %\quad 
    \(\sk(P)=\sk(G)\) \quad and \(\quad i\sim k\sim j\) is a collider \(\iff\) \(\forall\) \(C \subseteq \vij\) s.t.\ \(i\ci j\cd C\), we have \(k\not \in C\).
    \item \(\mathcal{V}_2(P,G)=\top\), 
    \quad
    if
    \(\sk(P)=\sk(G)\) \quad and \(\quad i\sim k\sim j\) is a non-collider \(\iff\) \(\forall\) \(C \subseteq \vij\) s.t.\ \(i\ci j\cd C\), we have \(k\in C\). 
    \item \(\mathcal{V}_3(P,G)=\top\), 
    \quad
    if
    \(\sk(P)=\sk(G)\) \quad and \quad \(\exists\) \(C \subseteq \vij\) s.t.\ \(i\ci j\cd C\), we have \(\quad i\sim k\sim j\) is a non-collider \(\iff\) \, \(k\in C\). 
\end{enumerate}

\begin{remark}
    If \(P\) satisfies the \emph{restricted faithfulness} assumption w.r.t.\ 
    some graph \(G\), all the orientation rules would  give the same output as the orientation rule in the conservative PC algorithm \citep{rams}. In particular, PC orientation rule 3 \emph{is} the orientation rule in conservative PC. 
    
    Note that \(\mathcal{V}_3(P,G)=\top\) is implied if \(P\) satisfies the 
    %\ts{is this ever defined?---this is referred to twice, and needs to be defined, in which case you will probably need the ancestor notation, which should help you define ancestral?} 
    pairwise Markov property w.r.t.\ \(G\) and is thus very weak. However,  as highlighted in Section \ref{comp} a weak property need not imply a weak correctness condition for the corresponding algorithm. Thus, the condition required for conservative PC to be correct does not have to be  stronger or weaker than the condition for the PC variants to be correct; indeed this will be verified in Section \ref{condegs}.

    Note that for \(\mathcal{V}_3\), the quantifier \(\exists\) \(C \subseteq \vij\) s.t. \(i\ci j\cd C\) applies to both sides of the implication:
    \begin{center}
        \( i\sim k\sim j\) is a non-collider \(\iff\) \, \(k\in C\),  
    \end{center}
    thus distinguishing it from \(\mathcal{V}_1\) and \(\mathcal{V}_2\) where the quantifier only applies on one side of the implication. \erk
\end{remark}

%Each \(\mathcal{V}_i\) corresponds to the orientation rule \(i\) above. 

\begin{comment}
    Trivially,  for \(P\) and \(G\), we have
\begin{align*}
    \mathcal{V}_3(P,G)=\top \quad \Rightarrow\quad  \mathcal{V}_2(P,G)=\top \text{ and } \mathcal{V}_1(P,G)=\top.
\end{align*}
\end{comment}

For $I \in \ns{1,2}$, by Proposition \ref{Auniq}, we have that \(\mathcal{V}_I\)-uniqueness always holds, \(U_{\mathcal{V}_I}(P)=\top\), if \(\mathcal{V}_I\) holds, \(\mathcal{V}_I(P,G_P)=\top\) for some \(G_P\); thus taking the conjunction of \(\mathcal{V}_I\) with \(U_{\mathcal{V}_I}\) does not change \(\mathcal{V}_I\).%\ts{trivial for what?}

By Proposition \ref{Auniq}, each of the corresponding PC properties, $\mathcal{V}_1$, $\mathcal{V}_2$, and $\mathcal{V}_3$  are class properties.   Hence, by Theorem \ref{mainth}, we obtain the following sufficient and necessary conditions for each variant of the PC algorithm depending on the orientation rule used.
\begin{proposition}[Exact correctness conditions for PC]
\label{PCconds}
%Given the distribution \(P\) and the true causal graph \(G_0\), for \(i=1,2\),
Let $P$ be the distribution for a true causal graph $G_0$.  Let $\mathcal{V}_1,\mathcal{V}_2$, and $\mathcal{V}_3$ be the corresponding PC properties to the various PC orientation rules. 
%in particular, $\mathcal{V}_3$ corresponds to conservative PC.   
Then
%\begin{center}
\[  \mathcal{V}_I(P,G_0)=\top \iff  \text{ the variant of PC algorithm which uses orientation rule \(I\) is correct} \quad \text{for } I\in \ns{1,2}\]
%\end{center}
and
%\begin{center}
\[\mathcal{V}_3(P,G_0)=\top \text{ and } U_{\mathcal{V}_3}(P)=\top \iff \text{the conservative PC algorithm is correct.} \]
%\end{center}
\end{proposition}

It is known that the \emph{restricted faithfulness} condition is a sufficient correctness condition for the PC-algorithm \citep{rams}, but it appears that Proposition \ref{PCconds} is the first description of the sufficient and \emph{necessary} correctness conditions for the PC algorithm.

Section \ref{condegs} provides examples of distribution \(P\) and graph \(G_0\) such that the condition in Proposition \ref{PCconds} holds. 

%Recall that we have assumed the availability of a conditional independence oracle, in practice there is randomness in the conditional independence testing step of the PC algorithm. To account for this, simulated data from examples in Section \ref{condegs} is used to illustrate how Proposition \ref{PCconds} holds under randomness of conditional independence testing, which is presented in the appendix.

\bmsubsection*{Proposition \ref{PCconds} under randomness of conditional independence testing}

Recall that we have assumed the availability of a conditional independence oracle; however, in practice there is randomness in the conditional independence testing step of the PC algorithm. To account for this,
we simulate data that satisfy the exact conditions $\mathcal{V}_1$ and $\mathcal{V}_3$ and $U_{\mathcal{V}_3}$ in Proposition \ref{PCconds}, but not restricted faithfulness---the weakest sufficient condition for the PC algorithm \citep{rams}. The distribution \(P\) is sampled from the following SCM with the DAG \(G\) from Figure \ref{counterexg}.

\begin{figure}[h]
    \centering
    \includegraphics[scale=0.7]{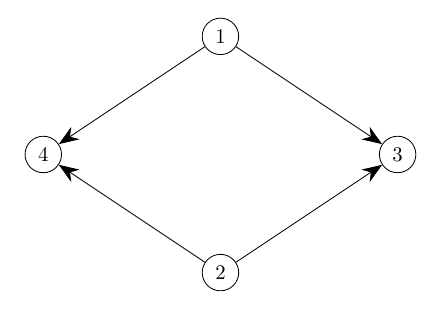}
    \caption{Counterexample graph for Examples \ref{v123ex}, \ref{exvuv3}, \ref{exa1a2}, \ref{exbub} and simulations for Proposition \ref{PCconds}.}
    \label{counterexg}
\end{figure}

\begin{align}
  \epsilon_i \stackrel{\text{i.i.d.}}{\sim} N(0,1), \text{ for }  
  i \in \{1,2,3,4\}\nonumber\\
  X_1 = \epsilon_1\nonumber\\
  X_2=  \epsilon_2\nonumber\\
  X_3=  2X_1-3X_2+\epsilon_3\nonumber\\
  X_4=   3X_1+2X_2+\epsilon_4\label{sem2}
\end{align}

\begin{comment}
It can be seen that \(P\) induces
\begin{align*}
         J(P)=\{3\ci 4,\quad 3\ci 4\cd \{1,2\},\quad 1\ci 2\}.
\end{align*}
\end{comment}

It is not difficult to verify that $P$ satisfies \ref{1-2-3-4} and \(\mathcal{V}_1(P,G)=\top\) and \(\mathcal{V}_3(P,G)=\top\) and \(U_{\mathcal{V}_3}(P)=\top\).

From 
the
SCM given by \eqref{sem2}, we sample $1200$ test batches of $10000$ data points each. Variants of the PC algorithm that 
use 
orientation rule \(1\) and \(3\) are provided by the \verb|causal-learn| package in Python \citep{python}, by setting the \verb|uc_rule| parameter of the \verb|pc| function to be \(0\) and \(2\), respectively.

The corresponding PC algorithm was ran on the simulated data and the percentage of correct outputs Markov equivalent to the DAG in Figure \ref{counterexg} is then recorded, as shown in Table \ref{pcres}.

\begin{table}[h]
    \centering
    \begin{tabular}{cc}
      \toprule % from booktabs package
      \bfseries PC orient 1 & \bfseries PC orient 3\\
      \midrule % from booktabs package
      97\% & 95\%\\
      \bottomrule % from booktabs package
    \end{tabular}
    \caption{Percentage of correct outputs of PC algorithm using orientation rule \(1,3\) (PC orient 1, PC orient 3), with simulated data from SEM \ref{sem2} as input.}\label{pcres}
\end{table}

\subsubsection{Relating PC conditions}\label{condegs}
Having obtained exact correctness conditions, %here 
we provide basic relations of these conditions to some existing causal learning correctness conditions in the literature. 
%In particular of the SP and Me-LoNS algorithms. 
These relations are summarised as follows. Let $P$ be a distribution and $G$ be a graph.

\begin{enumerate}[(1)]
    \item[(a)] \(P\) satisfies the correctness condition of Me-LoNS w.r.t.\ \(G\), which is denoted here as \(\mathcal{V}\) and \(U_{\mathcal{V}}\), where \(\mathcal{V}\) is the property:
    \begin{equation*}
\mathcal{V}(P,G)=\top \text{ if and only if } P \text { is adjacency faithful, and  V-OUS and collider-stable 
w.r.t.\
\(G\).}
\end{equation*}
    \item[(b)] \(P\) satisfies the SMR assumption w.r.t.\ \(G\).
    \item[(c)] \(\mathcal{V}_1(P,G)=\top\).
    \item[(d)] \(\mathcal{V}_2(P,G)=\top\).
    \item[(e)] \(\mathcal{V}_3(P,G)=\top\) and \(U_{\mathcal{V}_3}(P)=\top\).
\end{enumerate}
We have that (b) is implied by (e) and the remaining statements are mutually incomparable.

Via Examples \ref{v123ex} and \ref{ev1v2smr}, we see that all the correctness conditions in Proposition \ref{PCconds} are all mutually incomparable.   %\ts{since the numbers do not correspond to the numbers, perhaps change the labels to (a)?}

\begin{comment}
Figure \ref{rel1}.  \ts{can you use $\Longrightarrow$?}
\begin{figure}[!h]
    \centering
    \includegraphics[scale=1.2]{figure/rel.pdf}
    \caption{Relations for the consistency conditions in Proposition \ref{PCconds}.}
    \label{rel1}
\end{figure}
\end{comment}
%\begin{gather*}
    %\mathcal{V}_1(P,G)=\top \not \Rightarrow \mathcal{V}_2(P,G)=\top\quad \text{ and } \quad\mathcal{V}_2(P,G)=\top \not \Rightarrow\mathcal{V}_1(P,G)=\top.\\
    %\mathcal{V}_1(P,G)=\top \not \Rightarrow\mathcal{V}_3(P,G)=\top \text{ and } U_{\mathcal{V}_3}(P)=\top \quad \text{ and } \quad\mathcal{V}_2(P,G)=\top \not \Rightarrow\mathcal{V}_3(P,G)=\top \text{ and } U_{\mathcal{V}_3}(P)=\top.
%\end{gather*}

\begin{example}[Neither \(\mathcal{V}_1\) nor \(\mathcal{V}_2\) is implied by \(\mathcal{V}_3\) and \(U_{\mathcal{V}_3}\)]\label{v123ex}
For \(\mathcal{V}_1\), let \(G\) be the graph from Figure \ref{counterexg} and let the distribution  \(P\) induce 
\begin{align*}
    J(P)=\{1\ci 2 \cd 3,\quad  3\ci 4\cd \{1,2\},\quad  1\ci 2\}.
\end{align*}
Then \(U_{\mathcal{V}_3}(P)=\top\) since all v-configurations except \(1\sim 3\sim 2\) are constrained to be non-colliders in order to not violate \(\mathcal{V}_3\), leaving \(1\sim 3\sim 2\) to be constrained as a collider. However, we see that \(\mathcal{V}_1(P,G)\neq \top\) when considering 
the collider \(1\fra 3\fla 2\). Note that this example of \(P\) is not singleton-transitive \citep{zbMATH05645279}.

For \(\mathcal{V}_2\), again let \(G\) be the graph from Figure \ref{counterexg} and let the distribution \(P\) induce
%\begin{center}
         $$J(P)=\{3\ci 4,\quad 3\ci 4\cd \{1,2\},\quad 1\ci 2\}.$$
%\end{center}
Then we see that \(U_{\mathcal{V}_3}(P)=\top\) since \(1\sim 4 \sim 2\) and \(1\sim 3 \sim 2\) are constrained to be colliders in order to not violate \(\mathcal{V}_3\), leaving the remaining v-configurations to be constrained as non-colliders. However, we see that \(\mathcal{V}_2(P,G)\neq \top\) when considering the non-collider \(4\fla 1\fra 3\). \erk
\end{example} 
%

\begin{comment}
\begin{figure}[!h]
    \centering
    \includegraphics[scale=1.2]{figure/rel2.pdf}
    
    \caption{Relations for the consistency conditions in Proposition \ref{PCconds} and some existing conditions.}
    \label{rel2}
\end{figure}
\end{comment}
\begin{example}[Neither \(\mathcal{V}_1\) nor \(\mathcal{V}_2\) implies SMR or \(\mathcal{V}_3\) and \(U_{\mathcal{V}_3}\)]\label{ev1v2smr}
    Consider a  distribution \(P\) which induces
    \begin{align*}
        J(P)=\{1\ci 2\cd 3,\quad 1\ci 2\},
    \end{align*} and the following graphs \(G_1\) and \(G_2\):
    \begin{align*}
        G_1:  1\fra 3\fra 2\qquad G_2:  1\fra 3 \fla 2
    \end{align*}
    Then we have that \(\mathcal{V}_1(P,G_1)=\top\) and \(\mathcal{V}_2(P,G_2)=\top\). But, \(P\) is not SMR 
    w.r.t.\
    \(G_1\) nor \(G_2\),
    since \(G_1\) and \(G_2\) are both sparsest Markov graphs of \(P\) and they are not Markov equivalent.
    
    We also have \(\mathcal{V}_3(P,G_1)=\top\) and \(\mathcal{V}_3(P,G_2)=\top\) for non-Markov equivalent \(G_1\) and \(G_2\), thus  \(U_{\mathcal{V}_3}(P)\neq \top\). Since \(\mathcal{V}_1(P,G_1)=\top\) and \(\mathcal{V}_2(P,G_1)\neq\top\), \(\mathcal{V}_1\) does not imply \(\mathcal{V}_2\). Likewise, \(\mathcal{V}_2(P,G_2)=\top\) and  \(\mathcal{V}_1(P,G_2)\neq\top\), \(\mathcal{V}_2\) does not imply \(\mathcal{V}_1\). \erk
\end{example}

\begin{proposition}[\(\mathcal{V}_3\) and \(U_{\mathcal{V}_3}\) imply SMR]\label{a3smr}
    For all  $P$ and $G$, we have  
$$\mathcal{V}_3(P,G)=\top \text{ and } U_{\mathcal{V}_3}(P)=\top \Rightarrow P \text{ satisfies the SMR assumption w.r.t.\ \(G\)}.$$
\end{proposition}

\begin{example}[Neither \(\mathcal{V}_1,\mathcal{V}_2\) nor \(\mathcal{V}_3\) and \(U_{\mathcal{V}_3}\) is implied by SMR]
Consider the following example from \citet{UhlSP}.  Let   \(G\) be the graph given by following.  
\begin{center}
     \includegraphics[scale=0.7]{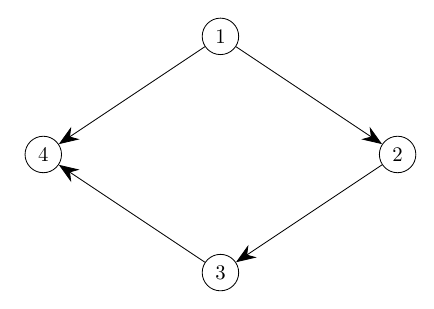}
     \end{center}
    Let \(P\) be a distribution which induces
     \begin{align*}
         J(P)=\{1\ci 3\cd 2,\quad 2\ci 4\cd \ns{1,3},\quad  1\ci 2\cd 4\},
     \end{align*}
     where \(P\) is SMR to \(G\) but adjacency faithfulness is violated for edge \(1\fra 2\).  \erk
\end{example}

\begin{example}[\(\mathcal{V}_3\) and \(U_{\mathcal{V}_3}\) are incomparable to \(\mathcal{V}\) and \(U_{\mathcal{V}}\)]\label{exvuv3}\hspace{2cm}
    \begin{itemize}
        \item
        (\(\mathcal{V}_3\) and \(U_{\mathcal{V}_3}\) do not imply \(\mathcal{V}\) and \(U_{\mathcal{V}}\)). Let \(G\) be the graph from Figure \ref{counterexg} and let the distribution \(P\) induce 
     \begin{align*}
         J(P)=\{3\ci 4,\quad 3\ci 4\cd \{1,2\},\quad 1\ci 2\}.
     \end{align*}Then we see that \(U_{\mathcal{V}_3}(P)=\top\) since \(4\sim1 \sim 3\) and \(4\sim 2 \sim 3\) cannot be colliders since the remaining v-configurations must be colliders in order to not violate \(\mathcal{V}_3\). However, V-OUS does not hold on the non-collider \(3\fla 1 \fra  4\).
        \item
        (\(\mathcal{V}\) and \(U_{\mathcal{V}}\) does not imply \(\mathcal{V}_3\) and \(U_{\mathcal{V}_3}\)). Let \(G\) be the graph from Figure \ref{counterexg} and let the distribution \(P\) induce 
     \begin{align*}
         J(P)=\{1\ci 2 \cd\{3,4\},\quad 3\ci 4\cd \{1,2\},\quad 1\ci 2\}.
     \end{align*}
     Then \(\mathcal{V}(P,G)=\top\) and \(U_\mathcal{V}(P)=\top\) since the 
v-configurations \(3\sim 2\sim 4\) and \(3\sim 1\sim 4\) cannot be colliders since the remaining v-configurations must be colliders in order to not violate the V-OUS property. However, \(U_{\mathcal{V}_3}(P)\neq\top\), since for non Markov equivalent graph \(G'\) with \(4\sim 1 \sim 3\) being a non-collider, \(\mathcal{V}_3(P,G')=\top\), thus \(U_{\mathcal{V}_3}(P)\neq \top\). \erk
    \end{itemize}
\end{example}

\begin{comment}
    \begin{proposition}[\(\mathcal{V}_3\) implies \(\mathcal{V}\) and \(U_{\mathcal{V}}\)]
\label{a3b}
    For all $P$ and $G$, we have $$\mathcal{V}_3(P,G)=\top\quad \Rightarrow \quad \mathcal{V}(P,G)=\top \text { and } U_\mathcal{V}(P)=\top.$$
\end{proposition}
\end{comment}

\begin{example}[Neither \(\mathcal{V}_1 \text{ nor } \mathcal{V}_2\) implies \(\mathcal{V}\) and \(U_{\mathcal{V}}\)]\label{exa1a2}
For \(\mathcal{V}_1\), let \(G\) be the graph from Figure \ref{counterexg} and let the distribution \(P\) induce 
%the following \(J(P)\):
     \begin{align}
     \label{1-2-3-4}
         J(P)=\{3\ci 4,\quad 3\ci 4\cd \{1,2\},\quad 1\ci 2\},
     \end{align}
      which we realize as an SCM in  \ref{sem2}.
      Then \(\mathcal{V}_1(P,G)=\top\), but \(\mathcal{V}(P,G)\neq \top\) due to V-OUS not holding on 
      the
      non-collider \(3\fla 1 \fra  4\).
     
     For \(\mathcal{V}_2\), let the distribution \(P\)  induce
     \begin{align*}
         J(P)=\{1\ci 2\cd 3,\quad  1\ci 2\}
     \end{align*}
    and let \(G\) be the graph given by  \(1\fra 3 \fla 2\). We  have that \(\mathcal{V}_2(P,G)=\top\) and \(\mathcal{V}(P,G)=\top\), but \(U_\mathcal{V}(P)\neq\top\), since for the non-Markov equivalent graph \(G'\) given by \(1\fra 3\fra 2\), we have \(\mathcal{V}(P,G')=\top\). \erk
\end{example}

\begin{example}[Neither \(\mathcal{V}_1 \text{ nor } \mathcal{V}_2\) is implied by \(\mathcal{V}\) and \(U_{\mathcal{V}}\)]\label{exbub}
For \(\mathcal{V}_1\), let \(G\) be the graph from Figure \ref{counterexg} and let the distribution \(P\) induce
%\begin{center}
         $$J(P)=\{1\ci 2 \cd \{3, 4\},\quad   3\ci 4\cd \{1,2\},\quad   1\ci 2\}.$$
%\end{center}
Then \(\mathcal{V}(P,G)=\top\) and \(U_\mathcal{V}(P)=\top\) since the 
v-configurations \(3\sim 2\sim 4\) and \(3\sim 1\sim 4\) cannot be colliders since the remaining v-configurations must be colliders in order to not violate the V-OUS property. However, we see that \(\mathcal{V}_1(P,G)\neq \top\) when considering the collider \(1\fra 3\fla 2\).

For \(\mathcal{V}_2\), again let \(G\) be the graph from Figure \ref{counterexg} and let the distribution  \(P\) induce 
\begin{align*}
    J(P)=\{3\ci 4 \cd 1,\quad 3\ci 4\cd 2,\quad  3\ci 4\cd \{1,2\},\quad  1\ci 2\}.
\end{align*}
Then \(\mathcal{V}(P,G)=\top\) and \(U_\mathcal{V}(P)=\top\) by similar reasoning. However, we see that \(\mathcal{V}_2(P,G)\neq \top\) when considering 
the
non-collider \(3\fla 1\fra 4\).\erk
\end{example}

\begin{remark}
    It is not difficult to explicitly construct these counterexamples as structural equation models %\citet{10.1214/aoms/1177732676} 
\citep{ zbMATH05645279}.  \erk
\end{remark}

\section{Applications to minimality}\label{degen}

\subsection{Comparing Correctness Conditions from Minimality Constraints}\label{mincompare}

Recall the design paradigm suggested by Theorem \ref{mainth}, in which the designer first selects a property \(\mathcal{A}\) and studies how strong the correctness condition of the corresponding algorithm (\(\mathcal{A}\) and \(U_{\mathcal{A}}\)) is, before finally designing the actual steps of the corresponding algorithm. Section \ref{comp} considers the correctness conditions of algorithms corresponding to all possible properties; here we will focus on properties \(\mathcal{A}\) that are variants of the minimality assumptions in Section \ref{assume} and compare the correctness condition of the corresponding algorithm.

First, we restrict our graph class \(\mathbb{G}\) to be DAGs, Proposition \ref{propforstlit}, which compares the properties by \emph{themselves}, can then be combined with Propositions \ref{reverse} and \ref{preserve} to obtain the following result for DAGs:
%and substitute in variants of the minimality assumptions introduced in Section \ref{assume} as the property in our framework. Proposition \ref{propforstlit} can then be combined with Propositions \ref{reverse} and \ref{preserve} to obtain the following result for DAGs:

\begin{proposition}[SMR is the weakest amongst minimality, for DAGs]
\label{SMR good dag}
    Let the  graph class \(\mathbb{G}\) be DAGs. For a probability distribution \(P\) and a DAG \(G\) on the same set of nodes $V$, consider the following properties:
    \begin{itemize}
        \item  \(\mathcal{M}_1(P,G)=\top\) if \(P\) is minimally Markovian w.r.t. \(G\).
        \item  \(\mathcal{M}_2(P,G)=\top\) if \(G\) is a sparsest Markov graph of \(P\).
        \item  \(\mathcal{M}_3(P,G)=\top\) if \(P\) is Pearl-minimal w.r.t.\ \(G\).
        \item  \(\mathcal{M}_4(P,G)=\top\) if \(P\) is causally minimal w.r.t.\ \(G\).
    \end{itemize}
    Then  \(\supp(\mathcal{M}_1)\subseteq \supp(\mathcal{M}_2)\) and \(\supp(\mathcal{M}_4)\subseteq \supp(\mathcal{M}_3)\subseteq \supp(\mathcal{M}_2)\).
\end{proposition}

\begin{remark}
    
    In the case of DAGs, although the corresponding property \(\mathcal{V}\) for Me-LoNS is stronger than \(\mathcal{M}_1\) (additionally requiring the V-OUS and collider-stable condition), neither \(\supp(\mathcal{V})\) nor \(\supp(\mathcal{M}_2)\) contain each other \cite[Example 2]{teh}; this does not contradict \(\supp(\mathcal{M}_1)\subseteq \supp(\mathcal{M}_2)\) in Proposition \ref{SMR good} since implication is in general not preserved after taking the conjunction with the corresponding uniqueness.
    
    Previously, \citet{lam} proved the containment \(\supp(\mathcal{M}_3)\subseteq \supp(\mathcal{M}_2)\); our proof of this containment is different, 
    %and  we obtain Proposition \ref{SMR good}---a similar version %for graph class \(\mathbb{G}\) being for \emph{maximal}  ancestral graphs (MAGs) \citep{10.1214/aos/1031689015},
    by appealing to the conditions in Propositions \ref{reverse} and  \ref{preserve} 
in Section \ref{theory}. 
\erk
\end{remark}

%In fact we can obtain a similar conclusion for graph class \(\mathbb{G}\) being \emph{maximal}  ancestral graphs (MAGs) \citep{10.1214/aos/1031689015}.

To obtain analogous results for MAGs, we first obtain an analogous result to Proposition \ref{propforstlit}, comparing the notions of minimality, but for graph class \(\mathbb{G}\) being maximal ancestral graphs (MAGs).

\begin{proposition}[Proposition \ref{propforstlit} for MAGs]\label{propforst}
    For a given distribution \(P\) and a maximal ancestral graph \(G\), the following statements imply the next:
    \begin{enumerate}
        \item \(P\) is minimally Markovian w.r.t. \(G\).
        \item \(G\) is a sparsest Markov graph of \(P\).
        \item \(P\) is Pearl-minimal w.r.t. \(G\).
        \item \(P\) is causally minimal w.r.t.\ \(G\).
    \end{enumerate}
\end{proposition}

Analogously substituting in variants of the minimality assumptions introduced in Section \ref{assume} as the property in our framework, Proposition \ref{propforst} is 
combined with Propositions \ref{reverse} and \ref{preserve} to obtain the following:

%\ts{should just state the result for both MAGs and DAGs, and argue for it in the proof}

\begin{proposition}[SMR is the weakest amongst minimality, for MAGs]
%
%, for MAGs]
\label{SMR good}
    Let the graph class \(\mathbb{G}\) be MAGs. 
 For a probability distribution \(P\) and a maximal ancestral graph \(G\) on the same set of nodes $V$, consider the following properties:
    \begin{itemize}
        \item  \(\mathcal{M}_1(P,G)=\top\) if \(P\) is minimally Markovian w.r.t. \(G\).
        \item  \(\mathcal{M}_2(P,G)=\top\) if \(G\) is a sparsest Markov graph of \(P\).
        \item  \(\mathcal{M}_3(P,G)=\top\) if \(P\) is Pearl-minimal w.r.t.\ \(G\).
        \item  \(\mathcal{M}_4(P,G)=\top\) if \(P\) is causally minimal w.r.t.\ \(G\).
    \end{itemize}
    Then  \(\supp(\mathcal{M}_1)\subseteq \supp(\mathcal{M}_2)\) and \(\supp(\mathcal{M}_4)\subseteq \supp(\mathcal{M}_3)\subseteq \supp(\mathcal{M}_2)\).
\end{proposition}

\begin{remark}\label{classremark}
Since 
the
notions 
of minimality in Section \ref{assume},  
except for the minimal Markov property,
depend on the graph class \(\mathbb{G}\), Proposition \ref{propforstlit} is not logically implied by Proposition \ref{propforst}, however the proof follows analogously. Likewise, since the corresponding uniqueness depends on the graph class, 
%\(\mathbb{G}\), 
Proposition \ref{SMR good dag} is not logically implied by \ref{SMR good}. \erk
\end{remark}

Propositions \ref{SMR good dag} and \ref{SMR good} imply that if one was to build an algorithm which outputs the set of all graphs that satisfy one of the minimality assumptions from Section \ref{assume} w.r.t. the input distribution \(P\), the algorithm corresponding to the sparsest Markov property, which happens to be the SP algorithm, is optimal amongst existing notions of minimality, in the sense that algorithms corresponding to other minimality notions would require stronger conditions to be correct.

From the following 
Proposition \ref{uniqpfaith}, 
the containment supp\((\mathcal{M}_4)\subseteq\) supp\((\mathcal{M}_3)\) in Propositions \ref{SMR good dag} and \ref{SMR good}, implies that for both DAGs and MAGs, \(\mathcal{M}_4\) and its corresponding uniqueness is stronger than or equivalent to faithfulness, however the following example shows that that this containment is strict.

\begin{example}[Causal minimality and its corresponding uniqueness is strictly stronger than faithfulness]
\label{cmfaith}
    Consider the following graphs:
    \begin{center}
        \includegraphics[scale=0.7]{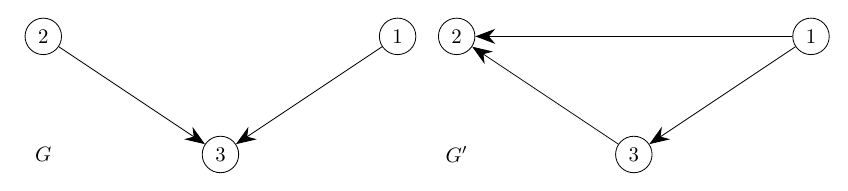}
    \end{center}
    Let a distribution \(P\) be faithful to the graph  \(G\).   Then \(P\) is causal minimal w.r.t.\ \(G\).  However, it it easy to verify that $P$ is also causal minimal w.r.t.\ the  non-Markov equivalent complete graph \(G'\). 
 Thus \(P\) is not causally minimal-unique.  \erk
\end{example}

\begin{comment}
    and the following Proposition \ref{uniqpfaith}, shows that \(\mathcal{M}_4\) and its corresponding uniqueness may potentially be strictly stronger than faithfulness;  to see this,  consider a non-maximal graph \(G\).   There is 
%always 
a Markov equivalent maximal graph \(G'\) such that \(G\) is a strict subgraph of \(G'\).  If \(P\) is faithful to maximal graph \(G'\), 
then it is cannot be  causally minimal to \(G' > G \).
%due to the existence of \(G\). 
\end{comment}

Markov equivalence characterisation results for MAGs from \citet{mecag} are used in proving Proposition \ref{propforst}, and applications of Proposition \ref{reverse} and \ref{preserve} are analogous for both DAGs and MAGs. As such, we conjecture that a version of Proposition \ref{propforst} for more general graph classes \(\mathbb{G}\) with known Markov equivalence characterisation such as maximal anterial graphs holds similarly, and applications of Propositions \ref{reverse} and \ref{preserve} follow analogously to obtain versions of Propositions \ref{SMR good dag} and \ref{SMR good} for larger graph classes \(\mathbb{G}\).

%Note that some containments from Proposition \ref{SMR good} exist in literature, such as \(\supp(\mathcal{M}_3)\subset \supp(\mathcal{M}_2)\) in \cite{lam}. \ts{Is this the only one, what other ones, they need to be cited, \textcolor{red}{I think that's it after some searching around, mostly from the CMU guys}} 

%\subsection{Additional knowledge than just Pearl-minimality is necessary beyond faithfulness}

\subsection{Necessity of Pearl-Minimality}
\label{pminimp}
Given any graph class \(\mathbb{G}\), recall that Theorem \ref{mainth} allows for any generic property \(\mathcal{A}\). However, Example \ref{A.ex} suggests that, for meaningful causal learning, the property \(\mathcal{A}\) selected at the start of the design paradigm suggested by Theorem \ref{mainth} should not be arbitrary.

\begin{example}[Degenerate \(\mathcal{A}\)]\label{A.ex}
    %Let property \(\mathcal{A}\) be defined as follows.  
    For each distribution $P$ fix an (arbitrarily) assigned  graph $G_P$; here, we can further specify that $P$ is not Markovian to $G_P$. Then for all $P$ and $G\in \mathbb{G}$, let $$\mathcal{A}(P, G)=\top \iff G \text { is Markov equivalent to } G_P.$$
    In this case there does not exist a property \(\mathcal{B}\) such that \(\supp(\mathcal{A})\subset \supp(\mathcal{B})\), thus the corresponding algorithm of the property \(\mathcal{A}\) is optimal in the sense that there is no other property that results in a corresponding algorithm with a strictly weaker correctness condition. However, an algorithm corresponding to the defined \(\mathcal{A}\), which returns the set of graphs \(\boldsymbol{G}(P)\) from input distribution \(P\) such that \(\mathcal{A}(P,\boldsymbol{G}(P))=\top\), simply returns the arbitrarily assigned MEC when defining \(\mathcal{A}\).\erk
\end{example}

The following examples suggest that any meaningful property in the context of causal learning has to at least satisfy the Pearl minimality property; i.e.\ for any meaningful property \(\mathcal{A}\), if \(\mathcal{A}(P,G)=\top\), then \(P\) is Pearl-minimal to \(G\).

%Causal statements such as conditional exchangeability and ignorability are conditional independence statements of the joint distribution of the interventional and observational marginals. \ts{how does the previous sentence lead into the latter} Thus 
In causal learning, the goal is to represent the conditional independence statements of the observational distribution \(P\) using the graphs in \(\boldsymbol{G}(P)\subseteq \mathbb{G}\). If \(P\) is not Markovian to all graphs in \(\boldsymbol{G}(P)\), then we have separations in some graphs of \(\boldsymbol{G}(P)\) implying additional conditional independencies that are not induced by \(P\), potentially resulting in wrong causal statements being made, thus showing that a meaningful property \(\mathcal{A}\) for causal learning has to at least satisfy the Markov assumption to avoid degeneracies, that is, if \(\mathcal{A}(P,G)=\top\), then \(P\) is Markovian to \(G\).

In the following example, when \(\mathcal{A}\) is just the Markov property, the correctness condition of the corresponding algorithm can be too strong.
\begin{example}
%[\(\mathcal{A}\) being the Markov assumption]
[Just the Markov property]

    Given any class of graphs \(\mathbb{G}\), let \(\mathcal{A}(P,G)=\top\) if \(P\) is Markovian to \(G\). Then \(\mathcal{A}(P,G)=\top\) and \(U_{\mathcal{A}}(P)=\top\) is equivalent to \(G\) being a complete graph and \(P\) not Markovian to any subgraph of \(G\).\erk
\end{example} 

Even if \(\mathcal{A}\) satisfies  the Markov property, the following degeneracy can still occur with the complete graph.

\begin{example}[Degenerate \(\mathcal{A}\) with the  Markov property]
    %Let property \(\mathcal{A}\) be defined as follows. 
    For all \(P\) and \(G\in \mathbb{G}\), let
    \begin{align*}
        \mathcal{A}(P, G)=\top \iff  \text{sk$(G)$ is complete}.
    \end{align*}
    Then there does not exist \(\mathcal{B}\) such that \(\supp(\mathcal{A})\subset \supp(\mathcal{B})\), thus the corresponding algorithm of property \(\mathcal{A}\) is again optimal in the sense that there is no other property that results in a corresponding algorithm with a strictly weaker correctness condition. However, the corresponding algorithm from the defined \(\mathcal{A}\), which returns the set of graphs \(\boldsymbol{G}(P)\) from input distribution \(P\) such that \(\mathcal{A}(P,\boldsymbol{G}(P))=\top\), simply outputs the set of graphs with complete skeletons for any input distribution. \erk
\end{example}

Although \(P\) is Markovian to graphs
%fully connected 
\(G\in \boldsymbol{G}(P)\) with a complete skeletons,  \(G\) does not imply any conditional independencies on \(P\); thus, no causal statements that are implications of conditional independencies on \(P\) can be made from \(G\), thus \(G\) is uninformative.
The output set of graphs \(\boldsymbol{G}(P)\) should aim to be as informative as possible, with all graphs in \(\boldsymbol{G}(P)\) having as many conditional independencies of \(P\) whilst satisfying the Markov property. Hence a meaningful property \(\mathcal{A}\) for causal learning should at least satisfy the Pearl-minimality constraint, that is,  if \(\mathcal{A}(P,G)=\top\),  
then \(P\) is Pearl-minimal to \(G\). This can also be expressed as: \(\mathcal{A}\) is at least as strong as Pearl-minimality.

\begin{remark}
    %From
    From Propositions \ref{propforstlit} and \ref{propforst}, it can be seen that causal minimality is weaker than faithfulness, which may suggest that a meaningful property \(\mathcal{A}\) for causal learning should at least satisfy causal minimality; i.e. if \(\mathcal{A}(P,G)=\top\), then \(P\) is causally minimal to \(G\). However, from Propositions \ref{SMR good dag} and \ref{SMR good} and Example \ref{cmfaith}, when \(A\) is just causal minimality, after taking conjunction with the corresponding uniqueness, the correctness condition of the algorithm corresponding to property \(\mathcal{A}\) is strictly stronger than faithfulness. Since we are interested in causal learning under correctness conditions weaker than faithfulness, we will not be considering causal minimality.
    %Hence, we do not consider causal minimality as a viable property constraint for causal learning beyond faithfulness. 
    \erk
\end{remark}

Similar arguments on minimality have also been considered by \citet{forster}.

\subsection{Necessity of Strengthening Pearl-Minimality to Relax Faithfulness}
In Section \ref{pminimp}, we deduced that any meaningful property \(\mathcal{A}\) for causal learning should be at least as strong as Pearl-minimality, which will be denoted as \(\mathcal{M}\). However, letting the property be just \(\mathcal{M}\), since \(\mathcal{M}\) is a class property, from Corollary \ref{maincor} we have that the \emph{exact} correctness condition of the corresponding algorithm is \(\mathcal{M}(P,G_0)=\top\) and \(U_{\mathcal{M}}(P)=\top\), which is faithfulness, as shown in Proposition \ref{uniqpfaith}. 

We say a distribution \(P\) is \emph{graphical} if there exists a graph \(G\) such that \(P\) is faithful to \(G\). 
\begin{proposition}[Pearl-minimal-unique and graphical distributions are equivalent]
\label{uniqpfaith}
    Let \(\mathbb{G}\) be a subclass of anterial graphs that contains DAGs and $\mathcal{M}$ be Pearl-minimality.  For a distribution \(P\),  it holds that \(U_{\mathcal{M}}(P)=\top\) if and only if \(P\) is graphical. 
\end{proposition}

\begin{remark}
    \citet{lam} proved that the set of distributions \(P\) such that \(U_\mathcal{M}(P)= \top\) is exactly the set of graphical distributions in the case of DAGs, and Proposition \ref{uniqpfaith} is a simple extension of this fact to the case of any class of graphs containing DAGs, including anterial graphs. \erk
\end{remark}

Thus, a property \(\mathcal{A}\) that is \emph{strictly stronger} than Pearl-minimality is \emph{necessary} to design meaningful constraint-based causal learning algorithms that are correct under \emph{strictly weaker} conditions than faithfulness. Table \ref{tab:example} shows examples of properties \(\mathcal{A}\) that are strictly stronger than Pearl-minimality which result in a correctness condition \(\mathcal{A}\) and \(U_{\mathcal{A}}\) of the  corresponding algorithm which is weaker than faithfulness.

\begin{table}[h!]
\normalsize
  \centering
  \begin{tabular}{|c |c|}
    \hline
    \textbf{Property \(\mathcal{A}\) strictly stronger than Pearl-minimality} & \textbf{Correctness condition \(\mathcal{A}\) and \(U_{\mathcal{A}}\) weaker than faithfulness}\\
    \hline
    Sparsest Markov property
    & SMR \\
    \hline
    Adjacency faithfulness, V-OUS and collider-stability & 
    Adjacency faithfulness, V-OUS, collider-stability and modified V-stability\\
    \hline
     Adjacency faithfulness and orientation faithfulness &
     Adjacency faithfulness and orientation faithfulness\\
    \hline
  \end{tabular}
  \caption{Examples of properties that are strictly stronger than Pearl-minimality, with correctness conditions of the corresponding algorithms that are weaker than faithfulness.}
  \label{tab:example}
\end{table}

Intuitively, the strengthening of Pearl-minimality \(\mathcal{M}\) into property \(\mathcal{A}\) excludes some Pearl-minimal graphs. This makes \(U_{\mathcal{A}}(P)=\top\) hold for more distributions \(P\) than \(U_{\mathcal{M}}(P)=\top\), which may result in a weaker overall correctness condition after taking the conjunction with \(\mathcal{A}\). A natural question would then be: ``Do there exist properties \(\mathcal{A}\) strictly stronger than \(\mathcal{M}\) such that \(U_{\mathcal{A}}(P)=\top\) holds for all distributions \(P\)?''. This would then imply that the algorithm corresponding to \(\mathcal{A}\) is optimal in the sense that there is no other property that results in a corresponding algorithm with a strictly weaker correctness condition. The following examples show that if we restrict the property \(\mathcal{A}\) to be any of the stronger notions of minimality in Section \ref{assume}, 
then there exists some distribution \(P\) such that \(U_{\mathcal{A}}(P)\neq \top\). 
% This implies that the algorithm corresponding to the property \(\mathcal{M}\wedge \mathcal{E}\), with \(\mathcal{E}\) being any of the other notions of minimality, there are always input distributions such that the algorithm would not be consistent. will not have consistent outputs with respect to some input distributions. 

\begin{example}[\(U_{\mathcal{A}}(P)\neq \top\) for some \(P\), for DAGs]
\label{degenex}
    Given \(n\) nodes, consider graphs given by
    \begin{align*}
        G_1:  1\fra 2\fra 
        \cdots
        \fra n-1\fra n \quad \text{ and} \quad G_2:  1\fra 2\fra 
        \cdots
        \fra (n-2)\fra  
        (n-1)
        %\xrightarrow{}
        \fla 
        n.
    \end{align*} 
    Let \(J(P)\) be the set of conditional independencies implied be the union of  Markov assumptions from both \(G_1\) and \(G_2\); 
   \(P\) is Pearl-minimal to  both \(G_1\) and \(G_2\). 
   %however \(G_1\) and \(G_2\) are not Markov equivalent, as such MEC.

     Note that the distribution \(P\) is minimally Markovian w.r.t.\
     \(G_1\) and \(G_2\) and thus, by Proposition \ref{propforstlit}, it also satisfies all other notions of minimality. Since \(G_1\) and \(G_2\) are not Markov equivalent, \(U_{\mathcal{A}}(P)\neq \top\) for \(\mathcal{A}\) being any of the stronger notions of minimality in Section \ref{assume}.  \erk
\end{example}

%Thus we see that without expert knowledge, with \(\mathcal{A}\) satisfying any existing notion of minimality, such a \(P\) should not be included in supp\((\mathcal{A})\), since it is ambiguous whether \(\mathcal{A}(P,G_1)=\top\) or \(\mathcal{A}(P,G_2)=\top\), for non-Markov equivalent graphs \(G_1\) and \(G_2\). 
%

%Stronger background knowledge \(\mathcal{E}\) on the true causal graph \(G_0\) than just existing notions of minimality in \ref{sec2} is needed to exclude Pearl-minimal graphs such that \(U_{\mathcal{M\wedge E}}(P)=\top\). In Example \ref{degenex}, if an expert knows that colliders should not be present in \(G_0\), then the Pearl-minimal graph \(G_2\) can be excluded.   
 %\ts{sort of a run on sentence...here you are trying to say something in general, but I will just stick to the example, which illustrates your point}
%since the above ambiguity can now be resolved. 
%Since without additional constraints from expert knowledge of
%
%Restricting the class of graphs \(\mathbb{G}\) to be DAGs, the following example based on Example \ref{ev1v2smr} shows that with the Markov property, regardless of the number of nodes.

%(\supp(\mathcal{A})\), because \(U_\mathcal{A}(P)=\top\) in order for \(P\) to be included in \(\supp(\mathcal{A})\). 

 The following example shows that this is similarly the case for \(\mathbb{G}\) being maximal anterial graphs, letting property \(\mathcal{A}\) be any of the stronger notions of minimality in Section \ref{sec2}.

\begin{example}[\(U_{\mathcal{A}}(P)\neq \top\) for some \(P\), for maximal anterial graphs]
\label{stex}
    Consider the distribution \(P\) which induces 
    \begin{align*}
        J(P)=\ns{1\ci 2, \quad 1\ci 2\cd \ns{3,4}, \quad 3\ci 4, \quad 3\ci 4\cd \ns{1,2}},
    \end{align*} 
    and the non-Markov equivalent DAGs \(G_1\) and \(G_2\) given by 
    Figure \ref{fig:side_by_side_minipage}.
    \begin{figure}[htbp]
    \centering
    \begin{minipage}[b]{0.45\textwidth}
        \centering
        \includegraphics[scale=0.8]{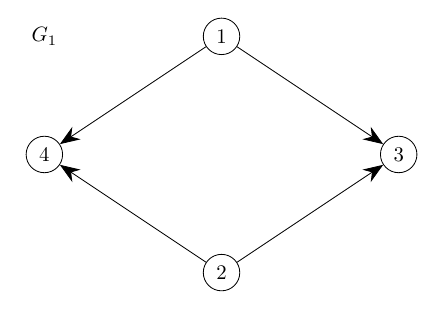}
    \end{minipage}
    \begin{minipage}[b]{0.45\textwidth}
        \centering
        \includegraphics[scale=0.8]{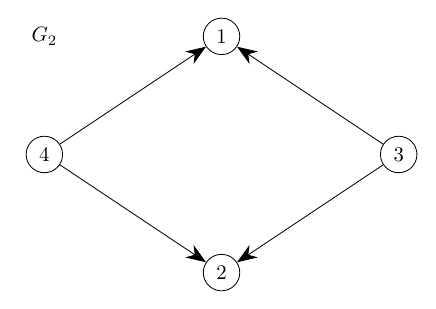}
    \end{minipage}
\caption{Non-Markov equivalant DAGs for Example \ref{stex}.}
\label{fig:side_by_side_minipage}
\end{figure}

There does not  exist a maximal
anterial graph \(G'\) such that \(J(G_1)\subset J(G')\subseteq J(P)\), thus \(P\) is P-minimal w.r.t. \(G_1\):    
Towards a contradiction, suppose there exists \(G'\) with  \(J(G_1)\subset J(G')\subseteq J(P)\).   Then \(J(G')\) will contain \(1\ci 2\) and \(3\ci 4\cd \ns{1,2}\) and either \(1\ci 2\cd \ns{3,4}\) or \(3\ci 4\).  
%\ts{are there two cases here, one corresponding to upward, and one corresponding to downward?} 
If \(3\ci 4\) is included, then ordered-upward and ordered-downward stability w.r.t.\  \emph{any} preorder  
\citep{orderfaith} must necessarily imply \(3\ci 4\cd 1\) or \(3\ci 4\cd 2\) which is not included in \(J(P)\), resulting in a contradiction; similarly if  \(1\ci 2\cd \ns{3,4}\) is included instead.  Thus, \(P\) is Pearl-minimal 
w.r.t.\  \(G_1\). The argument for Pearl-minimality of \(P\) w.r.t. \(G_2\) follows analogously.

It is clear that \(P\) is minimally Markovian and causally minimal to both \(G_1\) and \(G_2\). 

We see that both \(G_1\) and \(G_2\) are sparsest Markov graphs of \(P\), by showing there does not exist a maximal anterial graph \(G'\) such that \(|E(G')|< |E(G_1)|=|E(G_2)|=4\), and \(P\) is Markovian to \(G'\). If such a \(G'\) exists, then one of the edges in, without loss of generality, \(G_1\) will not be in \(G'\), which implies via the Markov property and maximality of \(G'\) a conditional independence of either \(1\ci 3 \cd C\) or 
\(1\ci 4 \cd C\) or \(2\ci 3 \cd C\) or \(2\ci 4 \cd C\) for some \(C\), a contradiction.

Thus \(U_{\mathcal{A}}(P)\neq \top\) for \(\mathcal{A}\) being any  other existing stronger notion of minimality in Section \ref{assume}, even in the class of maximal anterial graphs.
\erk
\end{example}

\begin{remark}
    In fact, in Example \ref{stex}, $P$ is a singleton-transitive compositional graphoid, an example of which is joint Gaussian distributions. 
    \erk
\end{remark}

Another way to strengthen Pearl-minimality \(\mathcal{M}\) is by including background knowledge \(\mathcal{E}\), resulting in the stronger property \(\mathcal{M}\wedge \mathcal{E}\).
Letting \(\mathcal{E}\) be the background knowledge such that no colliders are present in the graph, we can exclude the Pearl-minimal graph \(G_2\) in Example \ref{degenex}.

Given distributions from Examples \ref{degenex} and \ref{stex}, without background knowledge \(\mathcal{E}\), for property \(\mathcal{A}\) that are existing notions of minimality stronger than Pearl-minimality, the MEC of the output graphs for the corresponding algorithm for \(\mathcal{A}\wedge \mathcal{E}\) will not be consistent, the algorithm then has no reason to output one MEC over the other. Alternatively, the corresponding algorithm can simply output all the MECs for consideration.

\section{Summary and Discussion}
Consider any constraint-based causal learning algorithm that outputs a set of graphs \(\boldsymbol{G}(P)\) from an input distribution \(P\), we  represent the algorithm as a corresponding property \(\mathcal{A}\) as follows,
    \[\mathcal{A}(P,G)=\top \text{ for } G\in \mathbb{G} 
    \quad \iff \quad 
    G \in \boldsymbol{G}(P),\] 
Using this representation, this work contributes the following framework to study correctness conditions of any constraint-based causal learning algorithm.
 \[ \mathcal{A}(P,G_0)=\top \text{ and } U_{\mathcal{A}}(P)= \top\quad \Rightarrow \quad  \text{ the algorithm is correct}.\] 
 Using our framework, we recover some correctness conditions of causal learning in literature such as \citet{UhlSP} and \citet{teh}.  The framework implies that; given any causal learning algorithm, if we can identify the corresponding property \(\mathcal{A}\) that relates the set of output graphs \(\boldsymbol{G}(P)\) and input distribution \(P\), we can then immediately obtain correctness conditions for the algorithm, as \(\mathcal{A}\) and \(U_{\mathcal{A}}\). Under some cases, this correctness condition is exact. As an example, for any algorithm that uses an orientation rule to identify non-colliders and colliders in a skeleton to return DAGs;  by representing such orientation rules using local properties \(\mathfrak{n}\) and \(\mathfrak{c}\) on v-configurations,  we determine the corresponding property \(\mathcal{V}_{\mathfrak{n,c}}\). %constructed from local properties \(\mathfrak{n}\) and \(\mathfrak{c}\) on v-configurations. 
From this, based on given PC orientation rules and their corresponding property
 \(\mathcal{V}_{\mathfrak{n,c}}\), we use our framework to provide necessary and sufficient correctness conditions for the PC algorithm, depending on the version of orientation rules implemented in the actual computation, such as in \cite{python}.

Since the property \(\mathcal{A}\) describes both the output of the corresponding algorithm (using the term \(\mathcal{A}(P,\boldsymbol{G}(P))=\top\)) and the algorithm's correctness condition (using the terms \(\mathcal{A}(P,G_0)=\top\) and \(U_{\mathcal{A}}(P)=\top\)), this duality allows for any given property to be converted into the correctness condition of its corresponding algorithm by taking the conjunction with its uniqueness. Addressing the problem of correctness conditions of causal learning algorithms being too strong, this duality suggests the following paradigm to design causal learning algorithms. This paradigm allows for the correctness condition of an algorithm to be controlled for before designing the actual steps of the algorithm itself.
\begin{enumerate}
    \item Select property \(\mathcal{A}\).
    \item Study the correctness condition \(\mathcal{A}\) and \(U_{\mathcal{A}}\).
    \item Design the actual algorithm corresponding to \(\mathcal{A}\).
\end{enumerate}

Note that given any general property \(\mathcal{A}\), the framework does not contain information and steps on how to actually construct graphs \(G\) from a distribution \(P\) such \(\mathcal{A}(P,G)=\top\). This has to be designed separately.

%This paper contributes a framework to derive consistency conditions encompassing any constraint-based causal learning algorithm by appropriate substitution of a placeholder property, 

Given two selected properties, in Propositions \ref{reverse} and \ref{preserve}, we then give some conditions on the properties for when we can study and compare the strength of correctness conditions of the corresponding causal learning algorithms. Focusing on algorithms that output DAGs and MAGs, these conditions are used to show that if one were to construct constraint-based causal learning algorithms that correspond to existing notions of minimality %from 
\citep{forster, lam, zbMATH05645279}, the
SMR condition, which is the exact correctness condition for the SP algorithm \citep{UhlSP}, is the weakest possible correctness condition among these algorithms. This shows that the SP algorithm is optimal in the sense that there is no other existing notion of minimality that results in a corresponding algorithm with a strictly weaker correctness condition.

 We reason that Pearl-minimality is necessary for properties for meaningful causal learning that avoids degenerate cases, such as those in Section \ref{pminimp}. However, from the duality of the framework, having just Pearl-minimality as the property would result in the corresponding algorithm having faithfulness as the exact correctness condition.
%\ts{I think this should be presented as a however...the reasoning is one thing, but then it leads to faithfulness is the other....this is a good example of the duality....} we show that having only Pearl-minimality as the property would result in the corresponding algorithm having faithfulness as the exact correctness condition. 
Thus strengthening Pearl-minimality is necessary for constraint-based causal learning algorithms that are correct beyond faithfulness. This can be done using stronger properties such as those in Table \ref{tab:example}, or by including background knowledge \(\mathcal{E}\).
%This implies that the faithfulness assumption can only be relaxed by assuming any background knowledge on the true causal graph, perhaps solicited from experts or from performing interventions. 
Otherwise, motivated by Examples \ref{degenex} and \ref{stex}, our algorithm should just provide all the MECs as options for consideration. 

%Future work includes further \ts{where is this pre-develped?  an earlier version} developing the procedure of %outputting obtaining the Pearl-minimal graphs from \(J(P)\), perhaps by repeated applications of the contra-positive versions of the singleton transitive compositional graphoid axioms and ordered upward and downward stability w.r.t.\ a pre-order \citep{orderfaith}, to the conditional \emph{dependencies} of \(P\). However, some care has to be taken with selecting a pre-order since the Pearl-minimality constraint on \(J(G(P))\) may be violated without an appropriate order.

%Other avenues include investigating different mechanisms of how a meaningful property can arise, such as from \(G\) being the solution of a loss function optimisation, since these types of properties can be easily converted into an algorithm.  \ts{Is this worth putting?  Also, when you say other avenues, just sounds like we are listing off laundry list}

%\backmatter
%\bmsection*{Author contributions}

%\bmsection*{Acknowledgments}

%\bmsection*{Financial disclosure}

%\bmsection*{Conflict of interest}

%\bmsection*{Supporting information}

%Additional supporting information may be found in the
%online version of the article at the publisher’s website.

\section*{Appendix}
\bmsection*{Proofs}
%\section{Proofs}

%\bmsection{Proofs}

\begin{proof}[Proof of Theorem \ref{mainth}]\hspace{2pt}
\begin{itemize}
    \item [\(\Rightarrow\)]Given  the distribution \(P\),  for the output set of the algorithm \(\boldsymbol{G}(P)\), we have  \(\mathcal{A}(P,\boldsymbol{G}(P))=\top\).   For the true casual graph \(G_0\), we have \(\mathcal{A}(P,G_0)=\top\). Since \(U_{\mathcal{A}}(P)=\top\), we have \(G_0\) Markov equivalent to all graphs in \(\boldsymbol{G}(P)\).
    \item [\(\Leftarrow\)]Assume that \(\mathcal{A}\) is a class property and corresponds to the algorithm.   
    Let \(\boldsymbol{G}(P)\) be the output set of the algorithm.  
    We have \(\mathcal{A}(P,\boldsymbol{G}(P))=\top\), and all graphs in \(\boldsymbol{G}(P)\)
    %is
    are
    Markov equivalent to the true causal graph \(G_0\). Thus by definition of class property, we have \(\mathcal{A}(P,G_0)=\top\).
    Assume \(U_{\mathcal{A}}(P)\neq \top\), then we have that  there exists  a graph in \(\boldsymbol{G}(P)\) that is not Markov equivalent to \(G_0\), giving a contradiction. \qedhere
\end{itemize}
\end{proof}

\begin{proof}[Proof of Proposition \ref{reverse}]
    Let properties \(\mathcal{A}\) and  \(\mathcal{B}\), of which \(\mathcal{A}\) is a class property, be such that \(\mathcal{A}(P,G)=\top\Rightarrow \mathcal{B}(P,G)=\top\) for all \(P\) and \(G\).
    %Also let the distribution \(P\) be such that \(\mathcal{A}(P,G_P)=\top\) for some \(G_P\). 
    Thus each MEC on which \(\mathcal{A}\) is \(\top\) must also be \(\top\) for \(\mathcal{B}\), giving  the implication
    \begin{equation}
    \label{before}
    U_\mathcal{B}(P)=\top \Rightarrow U_{\mathcal{A}}(P)=\top.
    \end{equation}
    
    Let \(\mathcal{B}(P,G)=\top\)  for some \(P\) and \(G\), and \(U_\mathcal{B}(P)=\top\).
    Since 
    \begin{equation*}
        \exists G_P\in \mathbb{G} \text{ such that } \mathcal{A}(P,G_P)=\top\quad \iff \quad \exists G'_P\in \mathbb{G} \text{ such that } \mathcal{B}(P,G'_P)=\top,
    \end{equation*}
    we have \(\mathcal{A}(P,G')=\top\) for some \(G'\), thus we obtain \(\mathcal{B}(P,G')=\top\). Since \(U_\mathcal{B}(P)=\top\), we have that \(G\) is Markov equivalent to \(G'\), and by the class property of \(\mathcal{A}\), we have \(\mathcal{A}(P,G)=\top\); \(U_{\mathcal{A}}(P)=\top\) follows from \eqref{before}, as desired.
\end{proof}

\begin{proof}[Proof of Proposition \ref{preserve}]
    Let properties \(\mathcal{A}\) and  \(\mathcal{B}\), of which \(\mathcal{A}\) is a class property, be such that \(\mathcal{A}(P,G)=\top\Rightarrow \mathcal{B}(P,G)=\top\) for all \(P\) and \(G\).
    
    Let \(\mathcal{A}(P,G)=\top\) for some \(P\) and \(G, \) and \(U_{\mathcal{A}}(P)=\top\); 
    thus we obtain
    \(\mathcal{B}(P,G)=\top\). Since for distributions \(P\) such that \(\mathcal{A}(P,G_P)=\top\) for some \(G_P\), we have
    \begin{equation*}
         \{G \in \mathbb{G}: \mathcal{A}(P,G)=\top\}=\{G \in \mathbb{G}:  \mathcal{B}(P,G)=\top\},
     \end{equation*}
     it follows that \(U_{\mathcal{A}}(P)=\top \iff U_\mathcal{B}(P)=\top\).
\end{proof}

\begin{proof}[Proof of Corollary \ref{weaktoGP}]
The condition
\begin{align*}
        \exists G_P\in \{G: \mathcal{A}(P,G)=\top\} \text{ such that } \mathcal{E}(P,G_P)=\top,
\end{align*}
implies that the following condition from Proposition \ref{reverse}
\begin{equation*}
        \exists G_P\in \mathbb{G} \text{ such that } \mathcal{A}(P,G_P)=\top\quad \iff \quad \exists G'_P\in \mathbb{G} \text{ such that } (\mathcal{A\wedge E})(P,G'_P)=\top
    \end{equation*}holds for \(\mathcal{A\wedge E}\) in place of \(\mathcal{B}\). Proposition \ref{reverse} can then be applied to obtain \textnormal{supp}\((\mathcal{A})\subseteq \textnormal{supp}(\mathcal{A}\wedge\mathcal{E})\).
\end{proof}

To prove  Proposition \ref{Auniq}, we will use the following well known result: 
\begin{proposition}[\cite{pearl}]
\label{Markovchar}
   %Two  
   DAGs 
   %\(G_1\) and \(G_2\)
   are Markov equivalent if and only if they have the same skeletons and colliders. 
   %$\sk(G_1)=\sk(G_2)$ and $G_1$ and $G_2$ have the same colliders.
%
  %  \begin{enumerate}
   %     \item \(\sk(G_1)=\sk(G_2)\).
    %    \item The set of colliders in \(G_1\) coincides with the set of colliders in \(G_2\).
    %\end{enumerate}
 %
\end{proposition}

\begin{proof}[Proof of Proposition \ref{Auniq}]
\hspace{1cm}
  \begin{enumerate} 
    \item 
    Let \(G_1\) be Markov equivalent to \(G_2\) and \(\mathcal{V}_{\mathfrak{n,c}}(P,G_1)=\top\).   It suffices to verify that \(\mathcal{V}_{\mathfrak{n,c}}(P,G_2)=\top\).   The Markov equivalence gives $\sk(P) = \sk(G_1) = \sk(G_2)$, so that  Definition \ref{def-vnc}.\ref{item1-vnc} is satisfied.  Definition \ref{def-vnc}.\ref{item2-vnc} is also immediate from  Proposition \ref{Markovchar}.  
  %  
    %By Proposition \ref{Markovchar}, a collider \(i\sim k\sim j\) in \(G_1\) is a collider in \(G_2\), thus  \(\mathfrak{c}(P,i\sim k\sim j)=\top\). \(\mathcal{V}_{\mathfrak{n,c}}(P,G_1)=\top\) follows from a similar argument for non-colliders with \(\mathfrak{n}\).
%
%
\item
     Towards a contradiction, assume there exists non-Markov equivalent \(G_1\) and \(G_2\) such that \(\mathcal{V}_{\mathfrak{n,c}}(P,G_1)=\mathcal{V}_{\mathfrak{n,c}}(P,G_2)  = \top\). 
     %then we have 
     By Proposition \ref{Markovchar}, there exist some v-configuration \(i\sim k\sim j\) in \(\sk(G_1)=\sk(G_2)\), which is a collider in one graph, but not the other---without loss of generality, assume that it is collider in \(G_1\) and a non-collider in \(G_2\).
%
    %From the definition of \(\mathcal{V}_{\mathfrak{n,c}}\) 
    From Definition \ref{def-vnc}.\ref{item2-vnc},
    for \(G_1\), we have \(\frak{c}(P, i\sim k\sim j)=\top\) and for \(G_2\), we have \(\mathfrak{n}(P, i\sim k\sim j)=\top\), which contradicts \eqref{double-T}. \qedhere
    %However this contradicts the given condition on property \(\mathfrak{n}\) and \(\mathfrak{c}\). \qedhere
\end{enumerate}
\end{proof}

\begin{proof}[Proof of Proposition \ref{Achar}]
\hspace{2pt}
\begin{enumerate}
\item
This is immediate from De Morgan's laws.
%Let \(\mathfrak{n}\) and \(\mathfrak{c}\) be such that\begin{equation*}
   %     \mathfrak{n}(P,i\sim k\sim j)=\top \textnormal{ or } \mathfrak{c}(P,i\sim k\sim j)=\top.
  %  \end{equation*}
  %  Thus we see that \(\mathfrak{n}(P,i\sim k\sim j)\neq \top\) and \(\mathfrak{c}(P,i\sim k\sim j)\neq \top\) cannot both hold, thus the assignment is well-defined.
\item
\begin{enumerate}
    \item [] (b) \(\Rightarrow\) (a)
    Let \(\mathcal{V}_{\mathfrak{n,c}}(P,G)=\top\). %Since \(\sk(P)=\sk(G\)), it suffices to show that if a v-configuration \(i\sim k\sim j\) in \(\sk(P)\) is assigned to be a collider/non-collider by the \(\mathfrak{(n,c)}\)-orientation rule, then \(i\sim k\sim j\) is also a collider/non-collider in \(G\).  We have two cases:
\begin{enumerate}[(I)]
    \item If \(i\sim k\sim j\) is assigned to be a collider, then \(\mathfrak{n}(P,i\sim k\sim j)\neq\top\)---by Definition \ref{def-vnc}.\ref{item2-vnc}, it must be a collider in $G$.  
    %if \(i\sim k\sim j\) is a non-collider in \(G\), we have a contradiction by definition of \(\mathcal{V}_{\mathfrak{n,c}}\).
    \item 
    Similarly, if \(i\sim k\sim j\) is assigned to be a non-collider, then \(\mathfrak{c}(P,i\sim k\sim j)\neq\top\) and by Definition \ref{def-vnc}.\ref{item2-vnc} it must be a non-collider.
    %Likewise if \(i\sim k\sim j\) is assigned to be a non-collider \(k\in C\), then then \(\mathfrak{c}(P,i\sim k\sim j)\neq\top\), if \(i\sim k\sim j\) is a collider in \(G\), we have a contradiction by definition of \(\mathcal{V}_{\mathfrak{n,c}}\).
\end{enumerate}
        
\item[](a) $\Rightarrow$ (b) 
%Let for all v-configurations \(i\sim k\sim j\) in \(sk(P)\), assigned to be a collider/non-collider be colliders/non-colliders in \(sk(G)\) respectively, \(sk(G)=sk(P)\). Then 
For v-configurations \(i\sim k\sim j\) in \(G\), we have the following breakdown of  cases to verify Definition \ref{def-vnc}.\ref{item2-vnc}.
    \begin{enumerate}[(I)]
        \item If \(i\sim k\sim j\) is a collider in \(G\), then: %either:
        \begin{enumerate}
            \item If \(i\sim k\sim j\) is assigned as a collider, then  \(\mathfrak{n}(P,i\sim k\sim j)\neq\top\), and by assumption \(\mathfrak{c}(P,i\sim k\sim j)=\top\).
            \item If \(i\sim k\sim j\) is unassigned, then \(\mathfrak{n}(P,i\sim k\sim j)=\top\) and \(\mathfrak{c}(P,i\sim k\sim j)=\top\).
        \end{enumerate}
        \item If \(i\sim k\sim j\) is a non-collider in \(G\), then: %either:
        \begin{enumerate}
            \item If \(i\sim k\sim j\) is assigned as a non-collider, then \(\mathfrak{c}(P,i\sim k\sim j)\neq\top\), and by assumption \(\mathfrak{n}(P,i\sim k\sim j)=\top\).
            \item If \(i\sim k\sim j\) is unassigned, then \(\mathfrak{n}(P,i\sim k\sim j)=\top\) and \(\mathfrak{c}(P,i\sim k\sim j)=\top\).\qedhere
        \end{enumerate}
    \end{enumerate}
\end{enumerate}

\end{enumerate}

\end{proof}

\begin{proof}[Proof of Proposition \ref{PCconds}]
   % Orientation rule 1,2,3 are \(\mathfrak{(n,c)}\)-orientation rules with the following corresponding \(\mathfrak{n}\) and \(\mathfrak{c}\):
  We tabulate the orientation rule \(I\) with the following local properties  \(\mathfrak{n}\) and \(\mathfrak{c}\) obtained from negation, used to define the corresponding \(V_I\).
    \begin{table}[h]
    \large
        \centering
        \begin{tabular*}{500pt}{@{\extracolsep\fill}lcc@{\extracolsep\fill}}%
\toprule
\textbf{Orientation Rule \(I\)} & \textbf{\(\mathfrak{n}(P,i\sim k\sim j)=\top\) if} & \textbf{\(\mathfrak{c}(P,i\sim k\sim j)=\top\) if} \\
\midrule
1 & \(\exists\) \(C \subseteq \vij\) s.t. \(i\ci j\cd C\), we have \(k\in C\) & \(\forall\) \(C\subseteq \vij\) s.t. \(i\ci j\cd C\), we have \(k\not \in C\) \\
2 & \(\forall\) \(C\subseteq \vij\) s.t. \(i\ci j\cd C\), we have \(k \in C\) & \(\exists\) \(C\subseteq \vij\) s.t. \(i\ci j\cd C\), we have \(k\not\in C\) \\
3 & \(\exists\) \(C\subseteq \vij\) s.t. \(i\ci j\cd C\), we have \(k \in C\) &  \(\exists\) \(C\subseteq \vij\) s.t. \(i\ci j\cd C\), we have \(k \not \in C\) \\
\bottomrule
\end{tabular*}
        \label{tab:my_label}
    \end{table}

Let \(I \in \ns{1,2}\).
Since any v-configuration assigned to be collider/non-collider via the orientation rule \(I\) will also be a collider/non-collider in all graphs in the output set \(\boldsymbol{G}(P)\) of the corresponding PC algorithm, 
%and since the condition in Proposition \ref{Achar} holds, 
by Proposition \ref{Achar}, all graphs \(G\) that satisfy \(\mathcal{V}_{I}(P,G)=\top \) is in the output set \(\boldsymbol{G}(P)\);  
thus  \(\mathcal{V}_I\) corresponds to the algorithm and output set \(\boldsymbol{G}(P)\) satisfy \(\mathcal{V}_{I}(P,\boldsymbol{G}(P))=\top \). Applying Theorem \ref{mainth} gives us that \(\mathcal{V}_{I}(P,G_0)=\top\) %as the if condition 
is a sufficient condition for the correctness 
of the orientation rule \(I\), since for distributions \(P\) such that \( \mathcal{V}_{I}(P,G_P)=\top\) for some \(G_P\), we have \(U_{\mathcal{V}_{I}}(P)=\top\) by Proposition \ref{Auniq}; furthermore,  since \(\mathcal{V}_{I}\) is a class property and corresponds to the algorithm, we have that \(\mathcal{V}_{I}(P,G_0)=\top \)
is also a necessary condition. 

%The case of \(I=2\) follows similarly. 
The case for \(I=3\) also follows similarly except \(U_{\mathcal{V}_I}(P)=\top\) may not necessarily hold and needs to be assumed to apply   Theorem \ref{mainth}. 
%
%we obtain the if and only if condition of \(\mathcal{V}_I(P,G_0)=\top\) and \(U_{\mathcal{V}_I}(P)=\top\) instead.
\end{proof}

\begin{proof}[Proof of Proposition \ref{a3smr}]
    Via the contrapositive, if \(P\) does not satisfy the SMR condition, and if \(\sk(P)\neq \sk(G)\), then we are done. Suppose 
 \(\sk(P)=\sk(G)\). Then there exists \(G_1\) not Markov equivalent to \(G_2\) such that \(\sk(G_1)=\sk(G_2)\) and \(P\) is Markovian w.r.t. both \(G_1\) and \(G_2\). 
    By Proposition \ref{Markovchar},  we have for some v-configuration \(i\sim k\sim j\) in \(\sk(G_1)=\sk(G_2)\),  without loss of generality, \(i\sim k\sim j\) is a collider in \(G_1\) and a non-collider in \(G_2\). 
    
    From \(P\) being Markovian to \(G_1\), \(i\ci j\cd C\), for some \(C, k\not \in C\), and from \(P\) being Markovian to \(G_2\), \(i\ci j\cd C\), for some \(C, k \in C\).    Since \(V_3(P,G_1)=V_3(P,G_2)=\top\) for non-Markov equivalent \(G_1\) and \(G_2\), \(U_{\mathcal{V}_3}\) does not hold.
\end{proof}

%To prove Proposition \ref{propforst}, we use the following from \cite{mecag}, along with some definitions.

In our proof of Proposition \ref{propforst}, we will use a version of Proposition \ref{Markovchar} for MAGs due to \cite{mecag}.

\begin{definition}[Minimal collider path] A collider node on path \(\pi= \langle i_0,\ldots, i_n\rangle\) is a node \(i_m\) such that \(i_{m-1}\circ\fra  i_{m} \fla \circ i_{m+1}\), where 
(\(\circ\fra \)) indicates that the edge is either 
directed \((\fra \)) 
or
bidirected
(\(\arc\)). A subpath \(\pi_{\textnormal{sub}}\) between \(i_k\) and \(i_{\ell}\) of path \(\pi\) is a subsequence of \(\pi\), \(\langle i_k,\ldots, i_{\ell}\rangle\) such that \(\pi_{\textnormal{sub}}\) is a path.
    
    A \emph{collider path} is a path \(\langle i_0,\ldots, i_n\rangle\) such that\begin{enumerate}
        \item  \(i_0\) is not adjacent to \(i_n\),  
        \item \(i_m\) is a collider node on the path, for all 
        \(m \in \{1,\ldots,n-1\}\).
    \end{enumerate}
    A \emph{minimal collider path} is a collider path \(\langle i_0,\ldots, i_n\rangle \) such that any subpath between nodes \(i_0\) and \(i_n\) is not a collider path.
\end{definition}
\begin{proposition}[\citet{mecag}]\label{mecag}
    MAGs \(G_1\) and \(G_2\) are Markov equivalent if and only:
    \begin{enumerate}
        \item 
        $\sk(G_1) = \sk(G_2)$.
        \item \(G_1\) and \(G_2\) have the same minimal collider paths.
    \end{enumerate}
\end{proposition}

To prove Proposition \ref{propforst}, we will use discriminating paths, an analogue of colliders in MAGs.
\begin{definition}[Discriminating paths \citep{dispath}]
    A  path  \(\langle i_0, \ldots, i_{n-1},i_n\rangle\) is \emph{discriminating} if:
    \begin{enumerate}
        \item  the subpath \(\langle i_0,\ldots, i_{n-1}\rangle \) is a collider path, and  
        \item  \(i_m\fra  i_{n}\), for all  
        \(m \in\{1,\ldots,n-2\}\).
    \end{enumerate}
\end{definition}

\begin{proof}[Proof of Proposition \ref{propforst}]\hfill
    \begin{itemize}
        \item[] (\(1\Rightarrow 2\)) Suppose that  \(P\) is minimally Markovian to a
        maximal graph \(G\).  
        Let $P$ be Markovian to some MAG $G'$.  If $i$ is not adjacent to $j$ in $G'$, then by maximality, there is a separation in $G'$ 
        and 
        from the Markov property, there is a separation in $P$;  since $P$ is adjacency faithful to $G$, the node $i$ is also not adjacent to $j$ in $G$.   Thus we have 
        \begin{equation}
\label{subgraph}
i \text{ adjacent to } j \text{ in } \sk(G) \quad \Rightarrow \quad i \text{ adjacent to } j \text{ in }\sk(G')   \text{ for all MAGs } G' \text{ such that } P \text{ is Markovian to }G'.
        \end{equation}
        
\begin{comment}
        We have the following implications:
        \begin{align}
        %\label{subgraph}
            i \text{ adjacent to } j \text{ in sk}(G) \quad &\Rightarrow \quad i \text{ adjacent to } j \text{ in sk}(P) \nonumber
            \\
             &\Rightarrow \quad \forall C  (i\notci j\cd C)\quad \nonumber
             \\
             &\Rightarrow \quad 
            %i \text{ adjacent to } j \text{ in }\sk(G') \quad \forall \text{ MAG } G' \text{ s.t. } P \text{ is Markovian to }G'.
    \label{subgraph}
            \text{$i$ is adjacent to $j$ in $\sk(G')$, for all MAGs $G'$ such that $P$ is Markovian to $G'$}
        \end{align}
        where the first implication follows from minimal Markovian, since for a distribution \(P\) Markovian to a maximal \(G'\), we have:
        \begin{align*}
            i \text{ not adjacent to } j \text{ in } G'\quad \Rightarrow \quad \exists C (i\ci  j\cd C)
        \end{align*} the third implication is obtained from the contra-positive.
\end{comment}
        %
        %\eqref{subgraph}, %for all MAGs $G'$ such that $P$ is Markovian to $G'$, we have \(\sk(G)\) is a subgraph of \(\sk(G')\), thus \(|E(G)|\leq |E(G')|\), showing \(2\).
It follows from \eqref{subgraph} that if $P$ is Markovian to a MAG $G'$, then $\sk(G)$ is a subgraph of $\sk(G')$ so that  \(|E(G)|\leq |E(G')|\), as desired.
        \item[](\(2 \Rightarrow 3\)) We prove the contrapositive. 
  Let $G'$ be a MAG such that
        \begin{align} \label{p-minc}
            J(G)\subset J(G') \subseteq J(P). 
        \end{align}
%        
\begin{comment}
        We show that %necessarily
        \(|E(G')|<|E(G)|\), by arguing 
        that
        \(|E(G')|\geq|E(G)|\) causes a contradiction, thus showing 2 cannot hold.
\end{comment}
%
Towards a contradiction, suppose that \(|E(G')|\geq|E(G)|\).   We consider two cases.
\begin{itemize}
            \item[]  If \(|E(G')|> |E(G)|\), then we have 
            some \(i\) adjacent to \(j\) in \(G'\), but not in \(G\). By maximality, we have the separation \(i\perp_G j\cd C\) for some \(C\), but  \(i\not \perp_{G'} j\cd C\), contradicting \eqref{p-minc}.

            \item[]  If \(|E(G')|= |E(G)|\), then 
            $\sk(G) = \sk(G')$, 
            %otherwise 
            and we argue for an
             analogous contradiction, as  follows.  
            
            There must exist a minimal collider path in one graph, but not the other, otherwise \(J(G)=J(G')\) by Proposition \ref{mecag}, contradicting \eqref{p-minc}. 
            %Consider 
            Let
            \(\pi\)  be the shortest such minimal collider path in one graph, but not the other; furthermore, we assume that $\pi$ contains only one  non-collider node in the graph where \(\pi\) is not a collider path, since  if 
            %there is a minimal collider path 
            \(\pi\) contains more than one such node, we can take a shorter subpath \(\pi_{\text{sub}}\) with different endpoints on \(\pi\), and minimality would be preserved, otherwise \(\pi\) is not minimal.  Without loss of generality, let \(\pi\) be in \(G'\) but not in \(G\).
            
            Thus following similar reasoning as in the proof of Proposition 24 in 
        \citet{Sad}, it suffices to consider \(\pi\) to be a discriminating path of the form \(\langle i,i_1,\ldots,i_n,j\rangle \), with \(i_n\) being a collider node in \(G'\) 
        with \(n\geq 1\), while \(\langle i,i_1,\ldots, i_n\rangle \) is a collider path and \(i_n\) is a non-collider node in \(G\). %
            
            %The case where \(n=0\), implies that V-configurations \(i\sim k\sim j\) that are colliders/non-colliders in \(G\) must also correspond to colliders/non-colliders in \(G\), otherwise we have \(k\not \in C\) for all \(C\) such that \(i\ci j\cd C\in J(G)\) and \(k \in C\) for all \(i\ci j\cd C\in J(G')\), contradicting \ref{p-minc}.  Since colliders and non-colliders must coincide in both \(G\) and \(G'\), in \(G\), we have that v-configurations of the form \(i_{m-1}\sim i_{m}\sim i_{m+1}\) and \(i_m \sim j \sim i_{\ell}, |\ell-m|>2\) must be colliders. 

            For \(m\in \ns{1,\ldots, n-1}\), we have \(i_{m}\fra j\), otherwise a shorter minimal collider path than \(\pi\) would be created.   Thus 
            \(\langle i,i_1,...,i_n,j \rangle\) forms a discriminating path in \(G\) with \(i_n\) being a non-collider node. 
            Since \(i\) is not adjacent to \(j\) and via maximality of \(G\) and \(G'\), by \citet[Lemma 3.9]{dispath}, there exists \(C\) such that \(i\perp_G j\cd C\) and \(i_n \not \in C\), and there exists \(C'\) such that \(i\perp_{G'} j\cd C'\) and \(i_n \in C'\), contradicting \eqref{p-minc}.
            %
            %we have \(i_n \not \in C\) for all \(C\) such that \(i\perp_G j\cd C\) and \(i_n \in C\) for all \(i\perp_{G'} j\cd C\), contradicting \ref{p-minc}. 
            
        \end{itemize}

        \item[] \((3\Rightarrow 4)\) 
        We prove the contrapositive.  
        Let \(G'\) be a maximal subgraph \(G\), such that \(P\) is Markovian to \(G'\).  Then there are   nodes \(i\) adjacent to \(j\) in \(G\), but not in \(G'\).  Via maximality of \(G'\), we have  a separation in $G'$, but not in $G$. Since graphical separations are monotonic with respect to edge deletion, we have 
        %the separation \(i\perp_{G'} j\cd C \), but not in \(i\not \perp_{G} j\cd C \), 
        %and graphical separations in \(J(G)\) remains unchanged after edge deletion, thus 
        \(J(G)\subset J(G') \subseteq J(P)\), as desired. \qedhere
        %showing \eqref{p-minc} cannot hold. \qedhere 
    \end{itemize}
\end{proof}
\begin{proof}[Proof of Proposition \ref{propforstlit} (\(1\Rightarrow 2\))]
    Since DAGs are maximal, the proof follows similarly as the Proof of Proposition \ref{propforst} (\(1\Rightarrow 2\)), by replacing MAGs with DAGs.
    
The remaining implications are provided by \citet{forster}.
\end{proof}
\begin{comment}
\begin{proposition}[Minimally Markovian implies sparsest Markov]\label{propforst}
     For \ts{a} given distribution \(P\) and DAG \(G\), we have 
     $$P \text { minimally Markovian w.r.t. }G \Rightarrow G \text{ is the sparsest Markov graph of } P.$$
\end{proposition}
\begin{proof}
    Let \(G'\) be a sparsest Markov graph of \(P\), then by Proposition \ref{propforstlit}, \(G'\) is a causally minimal graph of \(P\). The skeleton of any causally minimal graph of \(P\) can be constructed from any permutation \(\sigma\) of 
 \(n\) graph nodes, as
 \begin{center}
     \(i\) adjacent to \(j \iff i\notci j\cd \{k:\sigma(k)<\min(\sigma(i),\sigma(j))\}\);
 \end{center} let \(\sk(G_\sigma)\) denote this skeleton.  Let \ts{check all the's and a's} the skeleton of the sparsest Markov graph \ts{two pieces of math next to each other, we don't know how to interpret} \(G', \sk(G')=\sk(G_{\sigma'})\) for some permutation \(\sigma'\).
 
 Compared with the construction of the skeleton \(sk(P)\), where \(i\) 
 adjacent to \(j \iff i\notci j\cd C\) for all \(C\), we see that if \(i\) is adjacent to \(j\) in \(sk(P)\), then \(i\) is adjacent to \(j\) in \(sk(G_\sigma)\), i.e. \(sk(P)\subseteq sk(G_\sigma)\) for all \(\sigma\), in particular \(sk(P)\subseteq sk(G_{\sigma'})\).

 If we have \(P\) minimally Markovian to \(G\), then \(sk(G)=sk(P)\subseteq sk(G_{\sigma'})=sk(G')\), and \(|G'|\leq |G|\) since \(G'\) the sparsest Markov graph of \(P\). So we have \(E(|G|)=E(|G'|)\), and \(G\) is a sparsest Markov graph of \(P\). 
\end{proof}
\end{comment}
%
%
%

\begin{proof}[Proof of Proposition \ref{SMR good dag}] 

     Let
       \(P\) be a distribution such that \(\mathcal{M}_1(P,G_P)=\top\) for some graph \(G_P\). 
     From \eqref{subgraph} it can be seen that all such \(G_P\) are subgraphs of sparsest Markov graphs and are thus sparsest Markov graphs. Hence the set of graphs \(G\) such that \(\mathcal{M}_1(P,G)=\top\) are exactly the set of sparsest Markov graphs, i.e. the set of graphs \(G\) such that \(\mathcal{M}_2(P,G)=\top\), satisfying the conditions in Proposition \ref{preserve}.

     Since \(\mathcal{M}_1(P,G)=\top\Rightarrow \mathcal{M}_2(P,G)=\top\) for all \(P\)  and \(G\) from Proposition \ref{propforstlit}, we can apply Proposition \ref{preserve} to obtain \(\supp(\mathcal{M}_1)\subseteq \supp(\mathcal{M}_2)\). 

\begin{comment}

    Since \(\mathcal{M}_2\) and \(\mathcal{M}_3\) are class properties, 
    and for all \(P\), and 
    for all distributions \(P\), there exists \(G^2_P\) such that \(\mathcal{M}_2(P,G^2_P)= \mathcal{M}_3(P,G^2_P)= \mathcal{M}_4(P,G^2_P)=\top\). Indeed, since for any distribution \(P\), there always exists a DAG  to which \(P\) is Markovian, \(G^2_P\) can be taken to be such a DAG with the least number of edges, and via Proposition \ref{propforstlit}, we have that \(\mathcal{M}_3(P,G^2_P)= \mathcal{M}_4(P,G^2_P)=\top\) as well.
    
    Thus from Proposition \ref{propforstlit}, we can apply Proposition \ref{reverse} to obtain \(\supp(\mathcal{M}_4)\subseteq \supp(\mathcal{M}_3)\subseteq \supp(\mathcal{M}_2)\).
\end{comment}

Let $P$ be a distribution. Note that  $P$ is  Markovian to a DAG $G_P$ with  a minimal  number of edges.  By Proposition \ref{propforstlit}, we have
\begin{equation}
    \label{nine-for-one}
\mathcal{M}_2(P,G^2_P)= \mathcal{M}_3(P,G^2_P)= \mathcal{M}_4(P,G^2_P)=\top.
 \end{equation}
 Since \(\mathcal{M}_2\) and \(\mathcal{M}_3\) are class properties, 
 Proposition \ref{reverse} with \eqref{nine-for-one} gives that  \(\supp(\mathcal{M}_4)\subseteq \supp(\mathcal{M}_3)\subseteq \supp(\mathcal{M}_2)\).
\end{proof}

\begin{proof}[Proof of Proposition \ref{SMR good}]
The proof is routine variation of the proof of Proposition \ref{SMR good dag}, where we apply  Proposition \ref{propforst} in lieu of Proposition \ref{propforstlit} to obtain \eqref{nine-for-one}.
\end{proof}

\begin{comment}
    \begin{proof}[Proof of Proposition \ref{a3b}]
    Consider the v-configuration \(i\sim k\sim j\) in \(G\), and \(\mathcal{V}_3(P,G)=\top\), then we see that \(\mathcal{B}(P,G)=\top\). Here we show \(U_\mathcal{B}(P)=\top\).

    By contrapositive if \(U_\mathcal{B}(P)\neq \top\), then we have non-Markov equivalent \(G_1, G_2\) such that \(\mathcal{B}(P,G_1)=\mathcal{B}(P,G_2)=\top\), then by Proposition \ref{Markovchar} and since \(sk(P)=sk(G)\), w.l.o.g. we have v-configuration \(i\sim k\sim j\) being a collider in \(G_1\) and not in \(G_2\). Then we have \(i\ci j\cd C\) for some \(C, k\not\in C\) from \(\mathcal{B}(P,G_1)\) and collider \(i\sim k\sim j\) in \(G_1\), and \(i\ci j\cd C\cup k\) from \(\mathcal{B}(P,G_2)\) and non-collider \(i\sim k\sim j\) in \(G_2\), this contradicts \(\mathcal{V}_3\).
\end{proof}
\end{comment}

\begin{proof}[Proof of Proposition \ref{uniqpfaith}]
 \hspace{2pt}
\begin{itemize}
    \item [\(\Leftarrow\)] 
    
    This direction is trivial.
    \item[\(\Rightarrow\)]
     Let \(P\) be such that \(U_{\mathcal{M}}(P)=\top\). 
     %and let \(J(G)\) 
     Let  \(G\in \mathbb{G}\) satisfy the Pearl-minimality constraint so that there is \emph{no} \(H\in \mathbb{G}\) such that
    \begin{align}\label{P-min cons}
        J(G)\subset J(H) \subseteq J(P).
    \end{align}
Towards a contradiction, if $G$ is not graphical, then      there exist some \(A\ci B \cd C \in J(P)\backslash J(G)\); via decomposition there exists some \(i\ci j\cd C \in J(P)\backslash J(G)\).  Via  \citet[Lemma B.14]{lam}, there exists a DAG \(G'\) such that \(J(G')=\{i\ci j\cd C\}\in J(P)\), from which we obtain a  \(H\) such that \(J(G')\subset J(H)\), so that
%and \(G''\)  
\eqref{P-min cons} holds. 
 \qedhere
%Since \(J(G'')\neq J(G)\), we contradict \(U_{\mathcal{M}}(P)=\top\). Hence \(J(P)\backslash J(G)\) must be empty, and \(P\) is graphical. \qedhere
 %   
    \end{itemize}
\end{proof}

\bibliography{wileyNJD-AMA}
\end{document}